\documentclass[twoside]{article}

\usepackage[accepted]{aistats2021}

\usepackage[utf8]{inputenc} %
\usepackage[T1]{fontenc}    %
\usepackage{hyperref}       %
\usepackage{url}            %
\usepackage{amssymb}%
\usepackage{pifont}%

\usepackage{booktabs}       %
\usepackage{graphicx}
\graphicspath{{../img/}}

\usepackage{threeparttable}
\renewcommand{\tnote}[1]{\textsuperscript{\textbf{#1}}}

\usepackage{wrapfig}
\usepackage{subfigure}

\usepackage[font=small]{caption}

\usepackage{amsfonts}
\usepackage{amsmath}
\usepackage{amsthm}
\usepackage{amssymb}
\usepackage{dsfont}
\usepackage{mathtools}
\usepackage{enumerate}
\usepackage{enumitem}
\usepackage{bbm}

\usepackage{xcolor}
\usepackage{color}
\usepackage{graphicx}
\usepackage{threeparttable}
\renewcommand{\tnote}[1]{\textsuperscript{\textbf{#1}}}
\usepackage{subfigure}

\usepackage{verbatim}
\usepackage{xspace} % for algorithm names

\usepackage{enumitem}
\usepackage{array}
\usepackage{multirow}
\usepackage{wrapfig}

\providecommand{\lin}[1]{\ensuremath{\left\langle #1 \right\rangle}}
\providecommand{\abs}[1]{\left\lvert#1\right\rvert}
  \providecommand{\R}{\mathbb{R}} % Reals
  \providecommand{\E}{{\mathbb E}}        %expectation
  \providecommand{\Eb}[1]{{\mathbb E}\left[#1\right] }       %expectation, with brackets
  \providecommand{\EEb}[2]{{\mathbb E}_{#1}\left[#2\right] } %expectation,  with brackets

\newcommand{\norm}[1]{\|#1\|}

\newcommand{\cN}{{\cal N}}

  \DeclareMathOperator*{\supp}{supp}

  \providecommand{\0}{\mathbf{0}}
  \providecommand{\1}{\mathbf{1}}
  \renewcommand{\aa}{\mathbf{a}}
  \providecommand{\bb}{\mathbf{b}}

  \providecommand{\ee}{\mathbf{e}}
  
  \let\ggg\gg
  \renewcommand{\gg}{\mathbf{g}}

  \providecommand{\uu}{\mathbf{u}}

  \providecommand{\xx}{\mathbf{x}}
  \providecommand{\yy}{\mathbf{y}}

  \providecommand{\mI}{\mathbf{I}}
  \providecommand{\mJ}{\mathbf{J}}

  \providecommand{\mM}{\mathbf{M}}

  \providecommand{\mP}{\mathbf{P}}

  \providecommand{\cN}{\mathcal{N}}
  \providecommand{\cO}{\mathcal{O}}

  \providecommand{\cS}{\mathcal{S}}

\usepackage{bm}
\newcommand{\bxi}{\boldsymbol{\xi}}

  \usepackage[textwidth=5cm]{todonotes}
  
\providecommand{\mycomment}[3]{\todo[caption={},size=footnotesize,color=#1!20]{\textbf{#2: }#3}}%
\providecommand{\inlinecomment}[3]{%
  {\color{#1}#2: #3}}%
\newcommand\commenter[2]%
{%
  \expandafter\newcommand\csname i#1\endcsname[1]{\inlinecomment{#2}{#1}{##1}}
  \expandafter\newcommand\csname #1\endcsname[1]{\mycomment{#2}{#1}{##1}}
}

\newtheorem{lemma}{Lemma}
\newtheorem{corollary}[lemma]{Corollary}

\newtheorem{definition}{Definition}

\newtheorem{assumption}{Assumption}
\newtheorem{property}[assumption]{Property}
\newtheorem{theorem}{Theorem}

\usepackage[capitalize,noabbrev]{cleveref}

\usepackage{url}

\usepackage{algorithm}
\usepackage[noend]{algpseudocode}

\errorcontextlines\maxdimen

\makeatletter
    \newcommand*{\algrule}[1][\algorithmicindent]{\makebox[#1][l]{\hspace*{.5em}\thealgruleextra\vrule height \thealgruleheight depth \thealgruledepth}}%
\newcommand*{\thealgruleextra}{}
\newcommand*{\thealgruleheight}{.75\baselineskip}
\newcommand*{\thealgruledepth}{.25\baselineskip}

\newcount\ALG@printindent@tempcnta
\def\ALG@printindent{%
    \ifnum \theALG@nested>0% is there anything to print
        \ifx\ALG@text\ALG@x@notext% is this an end group without any text?
        \else
            \unskip
            \addvspace{-1pt}% FUDGE to make the rules line up
            \ALG@printindent@tempcnta=1
            \loop
                \algrule[\csname ALG@ind@\the\ALG@printindent@tempcnta\endcsname]%
                \advance \ALG@printindent@tempcnta 1
            \ifnum \ALG@printindent@tempcnta<\numexpr\theALG@nested+1\relax% can't do <=, so add one to RHS and use < instead
            \repeat
        \fi
    \fi
    }%
\usepackage{etoolbox}
\patchcmd{\ALG@doentity}{\noindent\hskip\ALG@tlm}{\ALG@printindent}{}{\errmessage{failed to patch}}
\makeatother

\newbox\statebox
\newcommand{\myState}[1]{%
    \setbox\statebox=\vbox{#1}%
    \edef\thealgruleheight{\dimexpr \the\ht\statebox+1pt\relax}%
    \edef\thealgruledepth{\dimexpr \the\dp\statebox+1pt\relax}%
    \ifdim\thealgruleheight<.75\baselineskip
        \def\thealgruleheight{\dimexpr .75\baselineskip+1pt\relax}%
    \fi
    \ifdim\thealgruledepth<.25\baselineskip
        \def\thealgruledepth{\dimexpr .25\baselineskip+1pt\relax}%
    \fi
    \State #1%
    \def\thealgruleheight{\dimexpr .75\baselineskip+1pt\relax}%
    \def\thealgruledepth{\dimexpr .25\baselineskip+1pt\relax}%
}
\usepackage[round]{natbib}

\let\cite\citep
\definecolor{mydarkblue}{rgb}{0,0.08,0.45}
\hypersetup{ %
    colorlinks=true,
    linkcolor=mydarkblue,
    citecolor=mydarkblue,
    filecolor=mydarkblue,
    urlcolor=black%
    }

\begin{document}

\runningtitle{Critical Parameters for Scalable Distributed Learning}

\twocolumn[

\aistatstitle{Critical Parameters for Scalable Distributed Learning \\with Large Batches and Asynchronous Updates}

\aistatsauthor{ Sebastian U. Stich \And Amirkeivan Mohtashami \And  Martin Jaggi }

\aistatsaddress{ EPFL \And EPFL \And EPFL } ]

\begin{abstract}
It has been experimentally observed that the
efficiency of distributed training with stochastic gradient (SGD)
depends decisively on the batch size and---in asynchronous implementations---on the gradient staleness.
Especially, it has been observed that the speedup saturates beyond a certain batch size and/or when the delays grow too large.\\
We identify a data-dependent parameter that explains the speedup saturation %
in both these settings. Our comprehensive theoretical analysis, for strongly convex, convex and non-convex settings, unifies and generalized prior work directions  that often focused on only one of these two aspects. In particular, our approach allows us to derive improved speedup results under frequently considered sparsity assumptions.
Our insights give rise to theoretically based guidelines on how the learning rates can be adjusted in practice.
We show that our results are tight and illustrate key findings in numerical experiments.

\end{abstract}

\section{Introduction}
Parallel and distributed machine learning training techniques have gained significant traction in recent years.
A large body of recent work examined the benefits of parallel training in data centers~\cite{Dean2012:downpour,Goyal2017:large}
or when scaling the training to millions of edge devices in the emerging federated learning paradigm~\cite{McMahan2017:fedavg,
Kairouz2019:federated}.
However, many of these works reported diminishing efficiency gains when surpassing a certain critical level of parallelism. For instance, in mini-batch SGD~\cite{Robbins:1951sgd,zinkevich2010parallelized,Dekel2012:minibatch}, where the training is parallelized by evaluating a randomly sampled mini-batch of size $b$ each iteration, 
near-linear optimal scaling is only possible for moderate batch sizes in practice. Recent studies report data-set dependent critical batch sizes beyond which the speedup saturates~\cite{Dean2012:downpour,Goyal2017:large,Shallue2019:batchsize,lee2020data}.

This saturation is not surprising when considering the extreme case of very large batches (larger than the training data set size), in which case SGD reduces to deterministic gradient descent (GD). It is known that training with GD cannot be accelerated by evaluating more than one gradient in parallel~\cite{Arjevani2018:delayed}. This shows that the critical level of parallelism depends on the \emph{stochasticity} of the task. Recent works introduced notions to measure stochastic gradient diversity on empirical risk minimization problems~\cite{Yin2018:diversity,Sankararaman2019:confusion}.
In this work, we consider more general stochastic problems, refine their notions and provide new insights.

\citet{Dekel2012:minibatch} provide a concise analysis of mini-batch SGD, and argue that---theoretically---for optimal parallel speedup, the batch size should be chosen $\Theta \big(\frac{\sigma^2}{\epsilon}\big)$, where $\sigma^2$ is a uniform (global) upper bound on the stochastic noise, and $\epsilon > 0$ the target accuracy.
As a consequence, any \emph{constant} batch size allows for near linear speedup\footnote{
We define near linear speedup as $T(b,\epsilon)\leq 2 T(1,\epsilon)$, $\forall \epsilon > 0$, where $T(b,\epsilon)$ denotes the oracle complexity (number of stochastic gradient evaluations) of an algorithm with parallelism $b$ (for instance mini-batch SGD with batch size $b$) to reach a certain accuracy $\epsilon$ on the considered problem instance. The constant in the definition could be replaced by an arbitrary other constant larger than one.}
when the target accuracy is small enough. Analogous phenomena have been observed for asynchronous parallel methods~\cite[cf.][]{Chaturapruek2015:noise,Hannah2018:unbounded}.
However, these observations can often not be corroborated  in practice, where speedup saturates beyond certain batch size thresholds~\cite{Shallue2019:batchsize}.
Reasons for this discrepancy could be, that 
in the context of machine learning applications we need to consider moderate values of the training-accuracy $\epsilon$ only (approximately $n^{-1}$, where $n$ denotes the training data set size), and we cannot consider $\epsilon$ to be an arbitrarily small value~\cite{Bottou2010:sgd}. Moreover, the uniform upper bound on the noise might be a too conservative parameter, as for instance in the context of overparametrized problems the variance can vanish, i.e.\ $\sigma^2 \approx 0$ close to the optimum~\cite{Ma2018:interpolation}. 

Based on founded theoretical arguments, we show that the optimal batch size scales as $\cO \big(\frac{\sigma_\star^2}{\epsilon} + M \big)$, where $\sigma^2_\star$ is a bound on the variance close to stationary points only (and can be much smaller than the previously mentioned $\sigma^2$), and $M$ a parameter we define later. This explains the optimal batch size in the important low-accuracy regime and matches with practical findings in terms of speedup saturation (cf.\ Figure~\ref{fig:intro}) but also regarding optimal learning rate scaling (cf.\ Figure~\ref{fig:speedup}). \looseness=-1

Interestingly, our findings are not limited to parallelism induced by large batches alone, but they also apply to settings where parallelism is caused by staleness (delayed gradient updates) or asynchronity.

\paragraph{Contributions.}
We study a broad variety of parallel versions of SGD, including mini-batch SGD, delayed SGD~\cite{Arjevani2018:delayed} and asynchronous \textsc{Hogwild!}~\cite{Recht2011:hogwild} in a unified way, and derive convergence rates for strongly-convex, convex, and the important non-convex setting.
We identify a parameter that %
that allows a tight interpretation of critical scaling parameters.
In particular, we find that $\tau_{\rm crit} := \cO\bigl(\frac{\sigma_\star^2}{\epsilon} + M \bigr)$, where $\sigma_\star^2$ and $M$ are data- and model-dependent constants, is a critical parameter that governs parallelism:
\begin{itemize}[nosep,leftmargin=12pt,itemsep=2pt]
 \item[--] We show that mini-batch SGD enjoys near-linear speedup up to a critical batch size $b_{\rm crit} = \tau_{\rm crit}$.
 As a practical guideline, our findings supports the widely-used \emph{linear scaling rule} for the learning rate, but only up to the critical batch size $b_{\rm crit}$.
 \item[--] For asynchronous and delayed SGD we show strong linear speedup if the delays are not larger as  $\tau_{\rm crit}$.
 \item[--] As a particular novel insight, we prove that 
 for problems with relative sparse gradient (measured by  parameter $\Delta \leq 1$), %
 a strong linear speedup can be attained as long as the delay (or batch size) $\tau = \cO (\Delta^{-1})$. This improves prior best results by a factor of $\Delta^{-1/2}$ and is tight in general.
 \item[--] We verify our findings in experiments and show that our identified parameters can explain speedup saturation observed in practice.
 We show this in a synthetic setup where we have tight control over the problem parameters, and further, we estimate the critical parameters
 on standard deep learning task. %
 \end{itemize}
\begin{figure}[t]
\centering
\includegraphics[width=.9\linewidth]{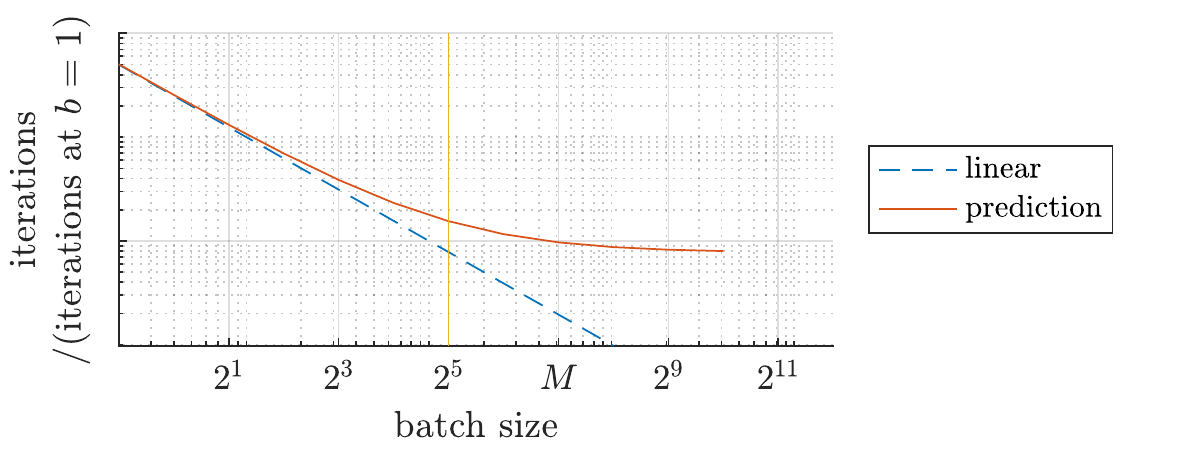}
\caption{Predicted parallel speedup $\cO\bigl(\frac{b+M}{b+bM}\bigr)$, see Section~\ref{sec:batch}. 
Relative number of iterations with batch size $b$ vs. iterations with batch size $b=1$, to reach the same target accuracy. The curve qualitatively matches with the empirical results of  \cite{Shallue2019:batchsize}.
} 
\label{fig:intro}
\end{figure}

\section{Related Work}
The seminal work of~\citet{Bertsekas1989:distributed} provides strong foundations for parallel and distributed optimization with stochastic algorithms and discusses asynchronous algorithms for optimization with several parallel workers---without providing non-asymptotic convergence rates. 

The parallel efficiency of asynchronous SGD methods was studied in \citep{Tsitsiklis1986:async}, using a parameter similar to our $M$ considered here.
Following the works of \citet{Zinkevich2009:slow,Recht2011:hogwild,Dean2012:downpour}, 
interest in the community renewed~\cite{Duchi2013:sparse,Mania2017:perturbed,%
Leblond2018:async,Nguyen2018:async,%
Arjevani2018:delayed,StichK19delays}, with particular focus on problems with sparse gradients (motivated by problems such as SVM, matrix completion, GLMs).
\citet{Agarwal2011:delayed} showed that under restrictive ordering assumptions delayed SGD updates can have negligible asymptotic effect. This observation was corroborated under much weaker assumptions in~\cite{Sa2015:bugwild,Chaturapruek2015:noise,Sra2016:adadelay,Nguyen2018:async}.

\citet{Dekel2012:minibatch} provide a concise analysis of mini-batch SGD and argue theoretically about the optimal batch size
and \citet{Friedlander2012:hybrid} propose exponentially increasing batch sizes on strongly-convex problems. Whilst these strategies yield near linear speedup, these schedules often do not align with practical needs, as discussed earlier. 
For constant batch sizes, it has been observed that linear speedup saturates beyond a certain threshold~\cite{Shallue2019:batchsize}, and several works aimed to express this saturation regime by data-dependent parameters, for instance by gradient diversity in~\cite{Yin2018:diversity,Chen2018:large} or as a function of the norm of the Hessian of the data in \citep{Jain2018:parallelSGD} for least square regression.
Whilst these results are very close to ours, and convey a similar message---that larger diversity in stochastic gradients allows for increased levels of parallelism---we extend their observations to more general settings.

The critical batch size is strongly linked the choice of the best learning rate. Our theorems suggest a learning rate scaling $\cO\big(\frac{b}{b+M}\big)$ which corroborates the popular linear scaling from~\cite{Goyal2017:large} but \emph{only for} $b\leq M$, and the learning rate should be kept constant when $b$ surpasses~$M$.

\section{Setup}
We now describe the 
theoretical framework. %

\subsection{Optimization Problem}
We consider the (stochastic) optimization problem
\begin{align}
 f^\star := \min_{\xx \in \R^d} f(\xx)  \,, \label{eq:prob}
\end{align}
where $f \colon \R^d \to \R$ is assumed to be $L$-smooth. %
\begin{assumption}[$L$-smoothness]
\label{ass:lsmooth}
A differentiable function $f \colon \R^n \to \R$ with gradients  
satisfying:
\begin{align}
 \norm{\nabla f(\xx)- \nabla f(\yy)} &\leq L \norm{\xx - \yy}\,, & & \forall \xx,\yy \in \R^d\,. \label{eq:lsmooth}
\end{align}
\end{assumption}

Sometimes  we will assume in addition that the objective function is convex.
\begin{assumption}[$\mu$-convexity]
\label{ass:convex}
A function $f \colon \R^n \to \R$ is $\mu$-convex if for all $\xx,\yy \in \R^d$:
\begin{align}
f(\yy) \geq f(\xx) + \lin{\nabla f(\xx),\yy-\xx} + \frac{\mu}{2}\norm{\yy-\xx}^2\,. %
\end{align}%
When $\mu > 0$ this is commonly known as $\mu$-\emph{strong-convexity}.
Our results can be extended to the weaker notion of Polyak-\L{}ojasiewicz condition.
\end{assumption}

\subsection{Stochastic Noise}
\label{sec:defM}
We assume that for every point in $\xx \in \R^d$ we can query a stochastic gradient $\gg(\xx)$ of $f(\xx)$, that is
\begin{align}
 \gg(\xx) := \nabla f(\xx) + \bxi(\xx) \,, \label{eq:goracle}
\end{align}
where $\bxi(\xx) \in \R^d$ denotes the realization of a zero-mean random variable. We do in general not assume that %
the noise is independent of $\xx$.
For $\epsilon \geq 0$, we define
\begin{align*}
 \sigma_\star^2 := \sup_{ \norm{\nabla f(\xx)}^2 \leq \epsilon } \E{ \norm{\bxi (\xx)}^2 } \,.
\end{align*}
For $\epsilon=0$, this measures the noise at stationary points, and for $\epsilon =\infty$ this recovers the standard notion assuming uniformly 
(globally) bounded noise on $\R^d$.
We further define\vspace{-3mm}
\begin{align*}
 M := \sup_{ \norm{\nabla f(\xx)}^2 > \epsilon } \frac{\E{ \norm{\bxi (\xx)}^2 }}{\norm{\nabla f(\xx)}^2} \,.
\end{align*}
We will drop the subscript in $\sigma^2$ %
whenever there is no ambiguity.
These two definitions imply:
\begin{property}[noise]\label{ass:noise}
There exist parameters $M \geq 0, \sigma^2 \geq 0$ such that for every gradient oracle as in~\eqref{eq:goracle} and for all $\xx \in \R^d$:
\begin{align}
 \Eb{ \bxi (\xx) } &= \0_d\,, &
 \E{ \norm{\bxi (\xx)}^2 }  &\leq M\norm{\nabla f(\xx)}^2 + \sigma^2\,. \label{eq:noise}
\end{align}
\end{property}

In related works, similar inequalities as~\eqref{eq:noise} are sometimes stated as an assumption~\cite{Tsitsiklis1986:async,Bottou2018:book}, e.g.\  for $M=0$ we recover the uniformly bounded noise assumption. Furthermore, it has been proved that this property always holds under certain assumptions~\cite{cevher2019linear}.

\subsection{Algorithm}
We now introduce an algorithmic template that can capture a broad class standard SGD implementations, such as mini-batch SGD, or asynchronous SGD. For simplicity, we assume a constant stepsize $\gamma$ throughout the iterations.
This formulation is identical to the description of the \textsc{Hogwild!}\ algorithm as for instance stated
 in~\cite{Mania2017:perturbed,Leblond2018:async}.\footnote{In the earlier version studied in~\cite{Recht2011:hogwild} only one single coordinate is updated per iteration (and the others discarded).} 
\algblock{ParFor}{EndParFor}%
\algnewcommand\algorithmicparfor{\textbf{for}}%
\algnewcommand\algorithmicpardo{\textbf{keep doing in parallel}}%
\algnewcommand\algorithmicendparfor{\textbf{end\ parallel loop}}%
\algrenewtext{ParFor}[1]{\algorithmicparfor\ #1\ \algorithmicpardo}%
\algrenewtext{EndParFor}{\algorithmicendparfor}%
\begin{minipage}{\linewidth} %
\begin{algorithm}[H]
		\caption{\textsc{parallel SGD template}} %
		\label{alg:hogwild}
		\resizebox{\linewidth}{!}{
		\begin{minipage}{1.1\linewidth}
		\begin{algorithmic}[1]
			\State {\bf Initialization}:
			shared variable $\xx = \xx_0 \in \R^d$
			\ParFor{$t=0,\dots,T$}
				\State $\xx_t \leftarrow$ inconsistent read of $\xx$ \label{ln:read}
				\State sample stochastic gradient $\gg_t := \gg(\xx_t)$ \label{ln:sample}
				\For{$v \in \supp(\gg_t) \subseteq [d]$}
				   \State $[\xx]_v \leftarrow [\xx]_v - \gamma [\gg_t]_v$ \hfill $\triangleright$ atomic coordinate write \label{ln:atomic}
				\EndFor
			\EndParFor
		\end{algorithmic}
		\end{minipage}}
\end{algorithm}
\end{minipage}

\paragraph{Special cases.} First, we remark that the standard mini-batch SGD algorithm with batch size $b \geq 1$ can be cast into the form of Algorithm~\ref{alg:hogwild}: Consider $b$ parallel processes, which all (consistently) read the state variable $\xx_t$, compute independent stochastic gradients $\gg_{t+i}$, for $i \in \{0, \dots, b-1\}$, and then apply the updates in a synchronous fashion, such that it holds $\xx_{t + b} = \xx_t - \gamma \sum_{i=0}^{b-1}\gg_{t+i}$. 

In addition, our framework also covers a broad range of asynchronous SGD implementations.
The parameter $\xx$ is allowed to inconsistently change during the read in line~\ref{ln:read} as other processes could be writing to~$\xx$ concurrently (line~\ref{ln:atomic}). As the processes also do not necessarily need to read (or write) the coordinates of $\xx$ in order (allowing for low-level system optimization) we have to be careful in the analysis with the definition of $\xx_t$. 

\paragraph{Global Ordering: ``After read'' approach.} We follow~\cite{Leblond2018:async} to define a global ordering of the iterates of Algorithm~\ref{alg:hogwild} and update the (virtual) counter $t$ \emph{after each complete read} of the shared variable $\xx$. A key property to be noted is that it holds
\begin{align*}
 \Eb{ \gg(\xx_t) \mid \xx_t } = \nabla f(\xx_t)\,,
\end{align*}
as the stochastic gradient is sampled only after $\xx_t$ is read completely. This might be obvious in our notation, though note that for instance in finite sum settings one might be tempted---for efficiency reasons---to sample an index $i \sim_{\rm u.a.r.} [n]$ \emph{before} reading $\xx$ and then only read the coordinates  that are relevant to compute $\nabla f_i(\xx)$. However, when using this shortcut, $\xx_t$ in general depends on the randomness used to generate the stochastic gradient and $\Eb{\gg(\xx_t))} \neq \nabla f(\xx_t)$ in general. See also~\cite{Leblond2018:async} for a thorough discussion of this issue.

A key assumption for our analysis is---as in prior work---that the writes on $\xx$ cannot overwrite $\xx$ arbitrarily, but only add or subtract values. 
\begin{assumption}[Atomic update]\label{ass:atomic}
The update of the coordinate $[\xx]_v$ on line~\ref{ln:atomic} is atomic.
\end{assumption}

In view of Assumption~\ref{ass:atomic} it follows that each iterate $\xx_t$ can be expressed as
\begin{align}
 \xx_t = \textstyle \xx_0 - \gamma \sum_{k=0}^{t-1} \mJ_k^t \gg_k
\end{align}
for diagonal matrices $\mJ_k^t \in \R^{d \times d}$, $k < t$, with 
\begin{align*}
 (\mJ_k^t)_{vv} = \begin{cases} 1 & \text{if $[\gg_k]_v$ written before $[\xx_t]_v$ was read,} \\ 0 & \text{otherwise.} \end{cases}
\end{align*}
Note that due to the concurrent nature of the writes of the processes to the shared vector, and by the fact that reads on line~\ref{ln:read} are not necessarily reading the coordinates in the same order, we can in general not assume $(\mJ_k^t)_{vv} \geq (\mJ_k^{t+1})_{vv}$. However, it is standard to assume bounded overlaps, i.e.\ a maximal delay during which iterations can overlap. This parameter captures the level of parallelism.

\begin{definition}[degree of parallelism]
 Define (with the convention that the maximum over the empty set is zero):
 \begin{align*}
  \tau := \sup_{t \geq 0} \max_{\substack{k < t,  \mJ_k^t \neq \mI_d}} \abs{t-k} + 1  \,.
 \end{align*}
\end{definition}
The parameter $\tau$ unifies common notions of parallelism: for instance in mini-batch SGD the parameter~$\tau$ is identical to the batch size $b$. For asynchronous methods with delays and staleness, the parameter $\tau$ is a uniform bound on the largest delay, recovering notions as in~\cite{Recht2011:hogwild,Leblond2018:async}.
While we do not investigate the mini-batch asynchronous setting explicitly, our theory also applies to the mini-batch asynchronous setting. In this case, the critical parameter would be, $b \cdot \tau$, i.e. the multiplication of the batch size and the delay.

\section{Main Results}
We now state our main convergence result.

\begin{theorem}\label{thm:main}
Let Assumptions~\ref{ass:lsmooth} %
and \ref{ass:atomic} hold,
let $(M,\sigma^2)$ denote parameters with Property~\ref{ass:noise},
and define the critical stepsize $\gamma_{\rm crit} := \frac{1}{10L(M+\tau)}$.
For any $\epsilon > 0$, there exists a stepsize $\gamma \leq \gamma_{\rm crit}$ 
such that 
Algorithm~\ref{alg:hogwild} reaches an $\epsilon$-approximate solution after at most the following number of iterations $T$:\\ 
\textbf{Non-Convex:} 
 $\min_{t \in [T]} \norm{\nabla f(\xx_t)}^2 \leq \epsilon$ after
\begin{align*}
\cO \left(\frac{\sigma^2}{\epsilon^2}  + \frac{M + \tau}{\epsilon} \right) \cdot L F_0
\end{align*}
iterations with $\gamma = \cO\big(\min\big\{\gamma_{\rm crit}, \big(\frac{F_0}{\sigma^2 T} \big)^{1/2} \big\}\big)$, where $F_0:= f(\xx_0) - f^\star$.\\
\textbf{Strongly convex:} If additionally Assumption~\ref{ass:convex} holds with $\mu > 0$, 
 then $\Eb{f(\bar \xx_T)-f^\star + \mu \norm{\xx_T- \xx^\star}^2} \leq \epsilon$ after
\begin{align*}
\tilde\cO \left( \frac{ \sigma^2}{\mu \epsilon} + \frac{L (M + \tau )}{\mu} \right) 
\end{align*}
iterations with $\gamma = \tilde\cO\big(\min\big\{\gamma_{\rm crit}, \frac{1}{\mu T} \big\} \big) $, ($\tilde \cO(\cdot)$ suppressing $\log$ factors) and\\
\textbf{Convex:} when $\mu=0$:
\begin{align*}
\cO \left(\frac{\sigma^2}{\epsilon^2} +  \frac{ L (M + \tau)}{\epsilon} \right) \cdot R_0^2 \,,
\end{align*}
with $\gamma = \cO\big(\min\big\{\gamma_{\rm crit}, \big(\frac{R_0^2}{\sigma^2 T} \big)^{1/2} \big\}\big) $, where $R_0^2 = \norm{\xx_0-\xx^\star}^2$.
Here $\bar \xx_T$ denotes a weighted average of the iterates $\xx_t$, $t \in \{0,\dots,T\}$ and $\xx_T$ the last iterate. 
\end{theorem}
The proof of this theorem follows from~\cite{StichK19delays} with only minor modifications of their proof. This earlier work did only consider the case when the degree of parallelism is exactly $\tau$ throughout the optimization and did not consider coordinate-wise overwrites.

For many special cases, Theorem~\ref{thm:main} recovers known convergence bounds. For instance for $M=\sigma^2=0$, the case of (deterministic) \emph{gradient descent}, it is well known that for a stepsize $\gamma = \frac{1}{L}$ the above convergence bounds can be reached, and in general not improved without acceleration techniques~\cite{Nesterov2004:book}. 
Similarly, for synchronous SGD with uniformly (globally) bounded noise ($M=0, \tau=1$), the dependency on~$\sigma^2$ can in general not be improved and matches known results~\cite{Nemirovsky:1983}. \looseness=-1

By considering the deterministic gradient descent setting, it is also clear that the critical stepsize cannot be significantly (up to constant factors) larger than $\frac{1}{L\tau}$, as for any batch size $b=\tau$, $\frac{1}{b}\sum_{i=0}^{\tau-1} \nabla f(\xx) \equiv \nabla f(\xx)$, and the stepsize in Algorithm~\ref{alg:hogwild} has to be scaled by $\frac{1}{b}$.

\section{Large Batch Training}
\label{sec:batch}
We will discuss these results in the following two sections. Whilst we focus in particular on large batch training in this section, and on asynchronous methods under sparsity assumption in the next section, our discussions are interchangeable, as we measure parallelism by a universal parameter $b=\tau$ and our results are not tied to a particular scheme.

\paragraph{Speedup and critical batch size.}
In Theorem~\ref{thm:main} we depict the oracle complexity $T(b,\epsilon)$, that is the number of gradient evaluations needed to reach a target accuracy $\epsilon$. In parallel implementations, for instance in mini-batch SGD with batch size $b$, we can gain (up to)\footnote{We ignore communication overheads in our discussion.} a factor of $b$ by computing $b$ gradients in parallel. Therefore, the parallel running time is $\frac{1}{b} T(b,\epsilon)$, and the \emph{parallel speedup} over a single thread implementation $T(1,\epsilon)$ is
\begin{align*}
 \frac{T(1,\epsilon)}{\frac{1}{b} T(b,\epsilon)} = b \cdot \cO \left( \frac{\sigma_\star^2/\epsilon + M + 1}{\sigma_\star^2/\epsilon + M + b}\right) \,,
\end{align*}
for all settings considered in Theorem~\ref{thm:main} (ignoring $L$). 
Here the first factor, $b$, indicates the potential linear speedup gained by the level of parallelism, and the second factor the slowdown from the increased number of required steps (gradient computations). We have \emph{near-linear} speedup when the second factor is bounded by a constant, for instance for any batch size not exceeding $b \leq \cO(1) \cdot b_{\rm crit}$, relative to the \emph{critical batch size} defined as \looseness=-1
\begin{align*}
b_{\rm crit} := \frac{\sigma_\star^2}{\epsilon} + M +1\,.
\end{align*}
As we are in particular interested in the low-accuracy regime, i.e. the case when $\sigma_\star^2$ is small or $\epsilon$ large~\cite[cf.][]{Bottou2010:sgd,Ma2018:interpolation}, the constant term $M$ is dominating in these bounds. In Figure~\ref{fig:speedup} (left) we illustrate this speedup value depending on $b$.

For the special case of deterministic problems, where $M=\sigma_\star^2 =0$, the critical batch size is $b_{\rm crit}= 1$ (as expected). Any level of parallelism increases the number of gradient computations linearly, as all parallel threads compute identical gradients.
On the other hand, for any stochastic problem with $\sigma_\star^2 > 0$, we see that the critical batch size can be unbounded. 
That is, stochastic problems can in principle be parallelized arbitrarily well in the asymptotic
regime~\cite[see also][]{Chaturapruek2015:noise,Hannah2018:unbounded,Nguyen2018:async}. However, as mentioned before, this regime might not be reached in practice.

\begin{figure}[t]
\centering
\hfill
\includegraphics[width=0.49\linewidth]{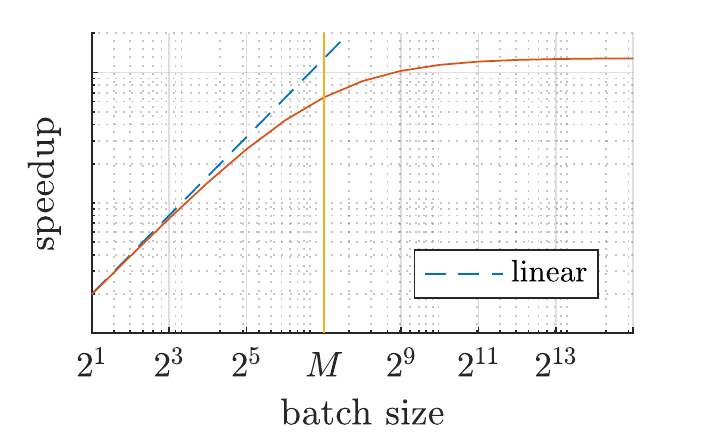}
\hfill
\includegraphics[width=0.49\linewidth]{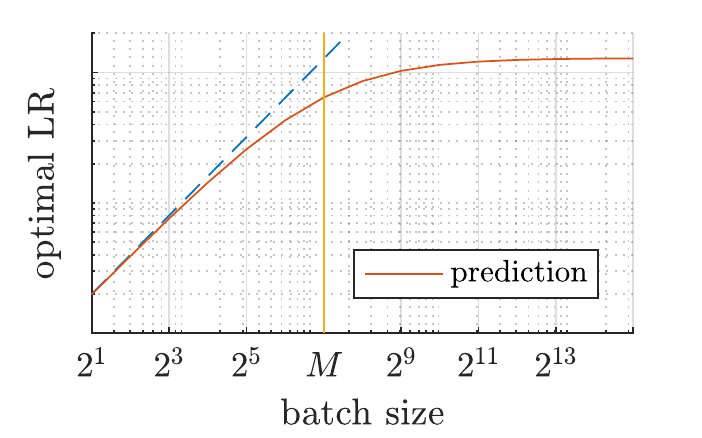}
\hfill\null\vspace{-1em}
\caption{Predicted parallel speedup $\cO\bigl(b \cdot \frac{M+1}{M+b}\bigr)$ %
in the low-accuracy regime (left). Optimal learning rate  (right).
} 
\label{fig:speedup}
\vspace{-4mm}
\end{figure}

\paragraph{Learning rate.}
The convergence results in Theorem~\ref{thm:main} depend crucially on the choice of the stepsize, i.e.\ the near-linear speedup can only be reached when the learning rate is optimally tuned.
In particular, for the low-accuracy regime, 
the learning rate should be chosen as large as possible for the fastest convergence, but smaller than the critical value, $\gamma_{\rm crit} = \cO \big(\frac{1}{L(M+b)}\big)$ that ensures convergence. Note that the critical stepsize is nearly constant for batch sizes $1 \leq b \leq M$ below the critical batch size.

In many implementations of mini-batch SGD, the batch gradients are 
 \emph{averaged} (opposed to just summation in Algorithm~\ref{alg:hogwild}):\vspace{-2mm}
\begin{align*}
 \xx_{t+b} := \xx_t - \frac{\gamma_{\rm mb}}{b} \sum_{i=0}^{b-1} \gg_{t + i}\,,
\end{align*}
thereby reducing the effective stepsize $\gamma = \frac{\gamma_{\rm mb}}{b}$ by a factor of $b$. However, in the linear speedup regime the effective steps size should not decrease $b$ fold, hence~$\gamma_{\rm mb}$ must be scaled by $b$ (\emph{linear scaling rule}).
This explains the linear scaling rule widely used in deep learning (but not learning rate warmup). Our theory also explains why the linear scaling does not apply beyond the critical regime $b \ggg b_{\rm crit}$. We illustrate this scaling in Figure~\ref{fig:speedup} (right).

\paragraph{Comparison to Gradient Diversity.}
\citet{Yin2018:diversity} introduced the notion of gradient diversity  for finite-sum structured problems to determine the critical batch size. As their parameter depends on the number of components (data points in the training data set) it cannot be extended to the stochastic setting considered here. However, they also consider a scaled version, the \emph{batch size bound}
\begin{align*}
B_\cS(\xx) := \frac{\E \norm{\gg(\xx)}^2 }{ \norm{\nabla f(\xx)}^2 } = 1 + \frac{ \E \norm{\bxi(\xx)}^2 }{ \norm{\nabla f(\xx)}^2 }
\end{align*}
which does not implicitly depend on the dataset size. 
Under our assumptions in Property~\ref{ass:noise}, and further assuming $\norm{\nabla f(\xx)}^2\geq \epsilon$, we see that $B_\cS(\xx) \leq b_{\rm crit}$, however, for points $\xx$ with $\norm{\nabla f(\xx)}^2 \leq \epsilon$ the value $B_\cS$ can be arbitrarily larger than $b_{\rm crit}$. Besides this difference, we observe that our critical batch size extends the notion of the batch size bound defined through gradient diversity only on empirical risk minimization problems to the more general class of stochastic problems.

\section{Relative Sparsity}

A line of work studied the speedup efficiency of SGD in terms of the (relative) sparsity of the stochastic gradients (we compare to these works in Table~\ref{tab:highlights}).
Whilst in deep learning settings the stochastic gradients are in general not sparse, such assumptions are well motivated for instance in applications with generalized linear models, where gradients follow the same sparsity patterns as the data.

\subsection{Speedup with sparsity}
For a vector $\xx \in \R^d$ let $\supp(\xx) \in 2^{[d]}$ denote the support of $\xx$, i.e.\ the set of coordinates where $\xx$ is non-zero. 
We now define a quantity that measures the sparsity of the stochastic gradients. Our definition generalizes the notion used in~\cite{Leblond2018:async} that was only defined for finite-sum structured objectives. \looseness=-1
\begin{definition}[$\Delta$-sparsity]\label{def:sparsity} 
Let $\mathds{1}_{X}$ denote the indicator function of the event $X$. Define $\Delta \leq 1$ as
\begin{align}
 \Delta :=  \sup_{ \xx \in \R^d} \max_{v \in [d]} \, \EEb{\gg(\xx)} {\mathds{1}_{v \in \supp(\gg(\xx))}}\,. \label{eq:delta}
\end{align}
In other words, $\Delta$ is a uniform upper bound on the probability that a given coordinate $v \in [d]$ is non-zero for a (random) stochastic gradient $\gg(\xx)$ at any $\xx \in \R^d$.
\end{definition}

\paragraph{Examples.}
A notable example are (random) coordinate descent methods, where we have $\Delta=\frac{1}{n}$, as every stochastic gradient is sparse.
However, note that our definition does not measure sparsity alone. For instance, for a problem defined as $f(\xx) := \frac{1}{2} [\xx]_1^2$ in ambient dimension $d$, with stochastic gradients $\gg(\xx) = ([\xx]_1 + u)\cdot \ee_1$, where $u \sim \cN(0,\sigma^2)$ is a Gaussian random variable, we have $\Delta = 1$, as the first coordinate is almost surely non-zero in every stochastic gradient.
For the special case of finite sum structured problems, $f(\xx) = \frac{1}{n}\sum_{i=1}^n f_i(\xx)$ our definition recovers the notion in~\cite{Leblond2018:async}. To see this, 
let $S_i := \cup_{\xx \in \R^d} \supp(\nabla f_i(\xx))$ denote the support of $\nabla f_i$. 
As in~\cite{Recht2011:hogwild,Leblond2018:async} we introduce $\Delta_r := \max_{v \in [d]} \abs{\{i \colon v \in S_i\}}$, the maximum number of data points with a specific feature (coordinate) and normalize $\Delta:= \Delta_r/n$. We observe that $\Delta$ by definition is an upper bound on the probability that a particular coordinate $v \in [d]$ is contained in the support $S_i$ of a $\nabla f_i$ chosen uniformly at random. Hence, in the special case of finite sum structured problems our Definition~\ref{def:sparsity} coincides with the literature~\citep{Recht2011:hogwild,Mania2017:perturbed,%
Leblond2018:async}.

\paragraph{Key observation.}
Whilst prior work utilized the sparsity assumption for refining and tightening inequalities that arise in the convergence proof of SGD,
we depart from this approach here. Instead, we show how improved convergence estimates directly follow from Theorem~\ref{thm:main}. For this, we observe that the sparsity correlates with the variance, i.e.\ high sparsity implies high variance.

\begin{lemma}\label{lem:trick}
Let $X \in \R$ be a real random variable, with $\Pr[X \neq 0] \leq \Delta $. Then $\abs{\E X}^2 \leq  \Delta \cdot \E \abs{X}^2$.
\end{lemma}
\begin{proof}
By direct calculation, we verify:
\begin{align*}
 \abs{\E X}^2 &= \abs{ \Pr[X=0] \cdot 0  + \Pr [X \neq 0] \cdot \E [X \mid X \neq 0 ] }^2 \\
 &=  \Pr [X \neq 0]^2 \cdot \abs{ \E [X \mid X \neq 0 ] }^2 \\
 &\leq \Pr [X \neq 0]^2 \cdot \E[ \abs{ X}^2 \mid X \neq 0 ] \\
 &= \Pr [X \neq 0] \cdot \left( 0 +  \Pr [X \neq 0] \cdot  \E[ \abs{ X}^2 \mid X \neq 0 ]  \right) \\
 &= \Pr [X \neq 0] \cdot \E \abs{X}^2 \leq \Delta \cdot \E \abs{X}^2\,,
\end{align*}
with Jensen's inequality and the assumption.
\end{proof}
\begin{corollary}
Let $\gg(\xx)$ be a stochastic gradient with $\E_{\gg(\xx)}= \nabla f(\xx)$ and with $\Delta \leq 1$ relative sparsity. Then
\begin{align*}
 \E_{\gg(\xx)} \norm{\gg(\xx)-\nabla f(\xx)}^2 \geq \left(\frac{1-\Delta}{\Delta} \right) \norm{\nabla f(\xx)}^2\,.
\end{align*}
\end{corollary}
\begin{proof}
Applying Lemma~\ref{lem:trick} coordinate-wise, we obtain
\begin{align}
 \norm{\nabla f(\xx)}^2 \leq \Delta \cdot \E_{\gg(\xx)} \norm{\gg(\xx)}^2 \,, \label{eq:delta}
\end{align}
and the claim follows %
by the bias-variance decomposition,
\begin{align*}
\E_{\gg(\xx)} \norm{\gg(\xx)-\nabla f(\xx)}^2 &= \E_{\gg(\xx)} \norm{\gg(\xx)}^2 - \norm{\nabla f(\xx)}^2 \\
&\stackrel{\eqref{eq:delta}}{\geq } \left( \frac{1}{\Delta} - 1\right)  \norm{\nabla f(\xx)}^2 \,. \qedhere
\end{align*}
\end{proof}

\paragraph{Consequences.}
For problems where the gradient norm $\norm{\nabla f(\xx)}^2$ in not uniformly bounded over $\R^d$, we conclude that it must hold $M \geq \frac{1-\Delta}{\Delta}$ in Property~\ref{ass:noise}. Hence, we see that we get linear speedup as long as $\tau \Delta \leq 1- \Delta$.
As highlighted in Table~\ref{tab:highlights}, our rates improve over the best previously known condition for speedup, $\Delta = \cO(\tau^{-1/2})$, from~\cite{Leblond2018:async}.

\begin{table}[ht]
\vspace{1em}
\caption{
Comparison of convergence bounds for asynchronous SGD with delay $\tau$. Early works relied on a bounded gradient assumption which was removed in~\cite{Nguyen2018:async}, showing sublinear $\cO\big(\frac{1}{\sqrt{\epsilon}} \big)$ convergence when $\sigma^2=0$, and~\citet{Leblond2018:async} are the first to show near linear speedup.
}
\label{tab:highlights}
\vspace{-1mm}
\begin{threeparttable}
\resizebox{\linewidth}{!}{
\begin{tabular}{ll} \toprule[1pt]
Asynchronous SGD reference &  Convergence Rate ($\norm{\xx_t-\xx^\star}^2 \leq \epsilon$) \\ \midrule
 & \multicolumn{1}{c}{\textit{\small bounded gradient assumption $\norm{\gg_t}^2 \leq G^2$}} \\
\citet{Recht2011:hogwild}  & $ \tilde\cO \left( \frac{L G^2 (1+\tau \rho + \tau^2 \Omega \sqrt{\Delta})}{\mu^2 \epsilon} \right)$\tnote{a}  \\
\citet{Sa2015:bugwild} & $ \tilde\cO \left( \frac{G^2}{\mu^2 \epsilon} + \frac{LG\tau}{\mu^2 \sqrt{\epsilon}} \right)$ \\
\citet{Chaturapruek2015:noise} & $\cO \left( \frac{G^2}{\mu^2 \epsilon} + C\right)$\tnote{b} \\
\citet{Mania2017:perturbed} & $ \tilde\cO \left( \frac{G^2 (1+ \tau \Delta)}{\mu^2 \epsilon} + \tau + \tau^2 \Delta \right) $  \\ \cmidrule{2-2}
\citet{Nguyen2018:async} & $ \tilde\cO \left( \frac{\sigma^2}{\mu^2 \epsilon} + \frac{\sqrt{(1+\sqrt{\Delta}\tau)(1+\tau)} (\sigma + LR_0)L}{\mu^2 \sqrt{\epsilon}} \right) $   \\ \cmidrule{2-2}
\citet{Leblond2018:async} & 
lin.-speedup for $\tau = \cO\big(\min\big\{ \frac{L}{\mu}, \frac{1}{\sqrt{\Delta} } \big\} \big)$ \\
\textbf{this paper } & lin.-speedup for $\tau = \cO\big(\frac{1}{ \Delta } \big)$ 
 \\ \bottomrule[1pt]
\end{tabular}
}
    \begin{tablenotes}[para]\footnotesize  %
\resizebox{0.73\linewidth}{!}{    
    \parbox{1\linewidth}{ \footnotesize 
     \item[a] $\rho \leq 1$ and $\Omega$ are additional parameters measuring sparsity of the gradients, see~\cite{Recht2011:hogwild}. %
     \item[b] $C=C(L,\mu,\tau)$ is an unspecified constant (asymptotic analysis only). 
     }
     }
    \end{tablenotes}
\end{threeparttable}
\end{table}

\textbf{Dependence on $\tau \Delta$ is best possible.}
We argue that the speedup condition $\Delta = \cO(\tau^{-1})$ cannot further be improved. To show this, we construct a problem instance for which SGD cannot achieve linear speedup if $\Delta \geq \omega(\tau^{-1})$ (i.e., asymptotically, $\Delta \tau \ggg 1$).

First, consider a $L$-smooth, $\mu$-strongly convex function $F \colon \R^d \to \R$. It is well known, that gradient descent with batch size $\tau$ cannot benefit from the parallelism and needs $\Omega(\tau)$ iterations in general to reach a  target accuracy $\epsilon$ (chosen sufficiently small). As argued earlier in Section~\ref{sec:batch}, this linear slowdown is expected, and not improvable.

Let us assume w.l.o.g.\ that~$\Delta$ is such that $B:= \Delta^{-1}$ is an integer, and for the dimension $D=B d$ define the block-separable function $f\colon\R^D \to \R$ as $f(\xx) = \sum_{i=1}^B F( [\xx]_{B_i})$, where $[\xx]_{B_i}$ denotes the projection of $\xx \in \R^D$ to the $i$-th block of $d$ coordinates. A $\Delta$-sparse, unbiased stochastic gradient of $f$ can be defined by $\gg(\xx) = B \cdot \nabla F( [\xx]_{B_i} ) \ee_{B_i}$ where $i \sim_{\rm u.a.r} [B]$ denotes a uniformly at random chosen block, and $\ee_{B_i} \in \R^D$ the indicator vector of this block.

Suppose now, that there exists a stepsize $\gamma$, such that SGD with batch size $\Delta \tau$ finds an $\epsilon$-approximate solution $\xx_T$ in $o(\Delta \tau)$ iterations. This means, that for each separable problem instance the condition $F( [\xx_T]_{B_i}) - F^\star \leq \epsilon$ could be reached with only $o(\tau)$ updates per block (in expectation). This is not possible in general, as argued above.

\subsection{Diversity-inducing mechanisms}
Following the observation that problems with high sparsity allow for  increased parallelism, one might wonder whether it is possible to accelerate training by artificially inducing sparsity. Such techniques where discussed in~\cite{Yin2018:diversity,Candela2019:sparsity}.
For instance, by artificially sparsifying stochastic gradients $\gg_t$ with a mask $\mM_\alpha$, $\E{\mM_\alpha}=\mI_d$, the stochastic gradients become sparse, with $\Delta = \alpha^{-1}$ for a tune-able parameter $\alpha \geq 1$~\cite{Alistarh2017:qsgd}.

Consider a problem where the baseline---non-sparsified SGD---converges as $\cO\big(\tau + 1\big)$, i.e.\ not allowing any parallelism beyond $\tau = \cO\big(1\big)$. With  unbiased sparsified gradients, SGD now enjoys the convergence bound $\cO\big(\tau + \alpha \big)$, tolerating parallelism up to $\tau = \cO\big(\alpha \big)$. However, when comparing this rate with the baseline result, we 
observe that even with parallelism $\tau$, there is no speedup that can be realized, as the total number of iterations increased by $\alpha$ when sparsifying the gradients. This means that our theoretical analysis presented here cannot confirm   the effectiveness of artificial sparsification as proposed in~\cite{Candela2019:sparsity} in general, though there is of course a possibility left that in special cases positive effects of sparsification can be observed in practice, or with modified versions of SGD~\cite{Alistarh2018:topk}.

\begin{figure*}
\centering 
\hfill
\hfill
\includegraphics[width=0.32\linewidth]{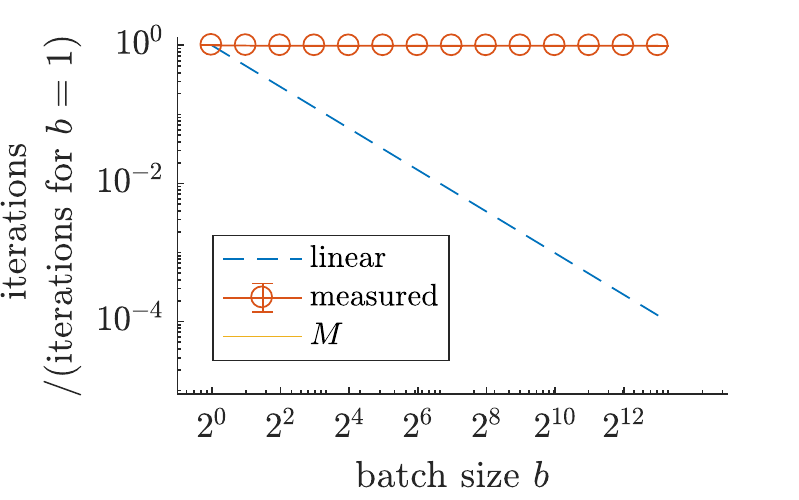}
\hfill
\includegraphics[width=0.32\linewidth]{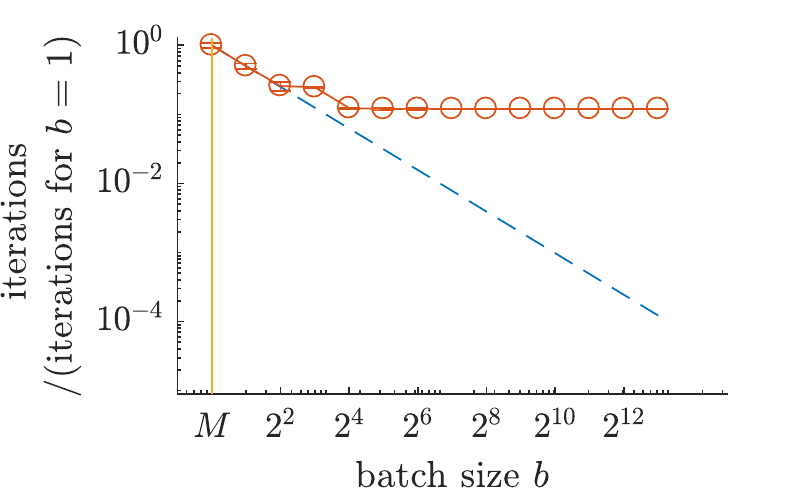}
\hfill
\includegraphics[width=0.32\linewidth]{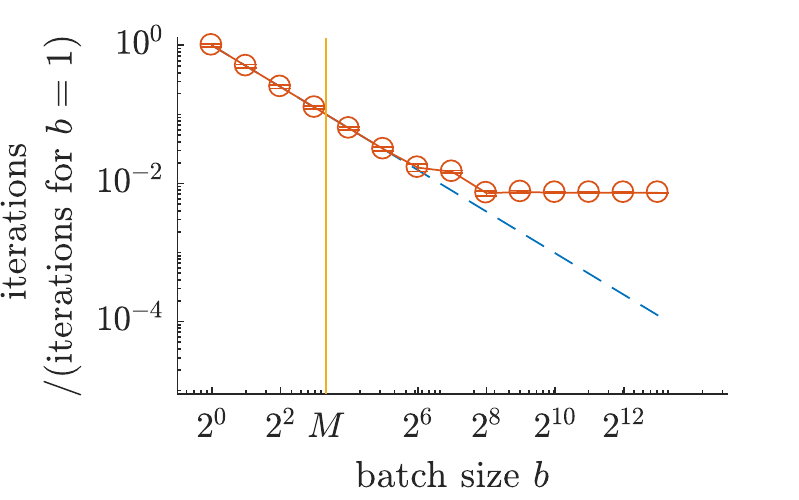}
\hfill\null\\
\hfill
\includegraphics[width=0.32\linewidth]{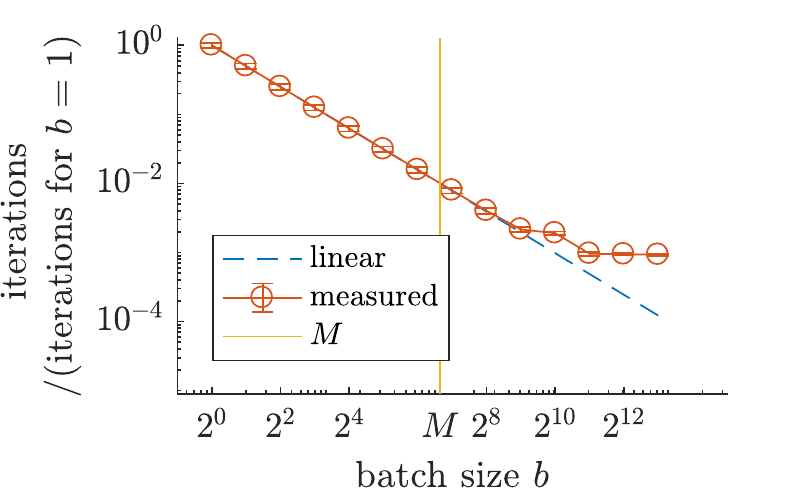}
\hfill
\includegraphics[width=0.32\linewidth]{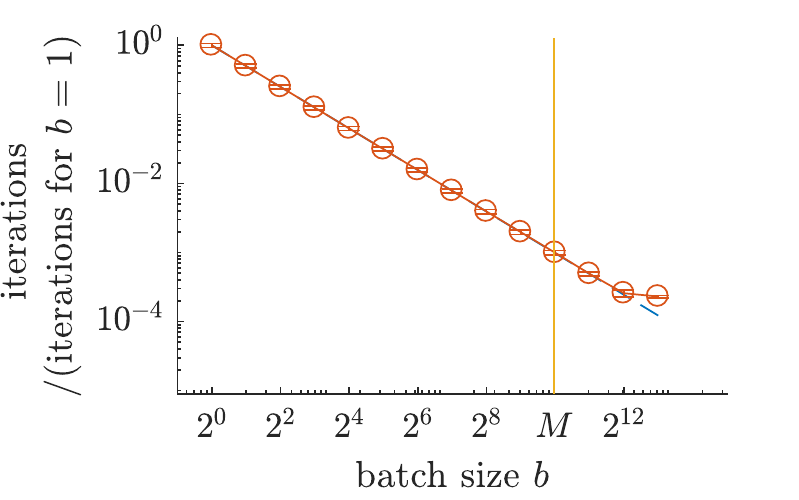}
\hfill
\includegraphics[width=0.32\linewidth]{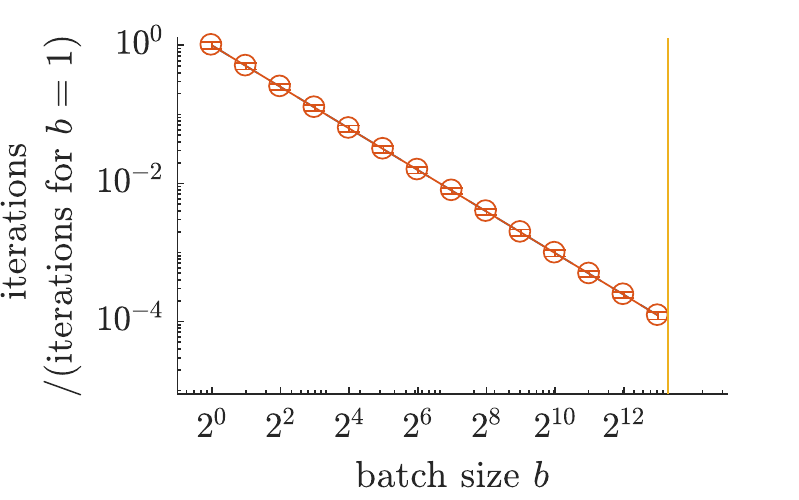}
\hfill\null
\caption{%
\textbf{Mini-batch SGD. Batch sizes smaller than $M$ enjoy linear speedup.}
Parallel speedup for various batch sizes $b \in \{2^0,\dots,2^{14}\}$ and problem instances with $M \in \{0,1,10\}$ (top) and $M \in \{100,1000,10000\}$ (bottom), on the synthetic optimization problem described in Section \ref{ssec:synth}. Plots depict number of iterations (i.e.\ parallel running time $\frac{1}{b} T(b,\epsilon)$), normalized by $T(1,\epsilon)$, required to reach the target accuracy with tuned optimal learning rates.
}
\label{fig:new}
\end{figure*}

\begin{figure*}[h!]
\hfill
\includegraphics[width=0.32\linewidth]{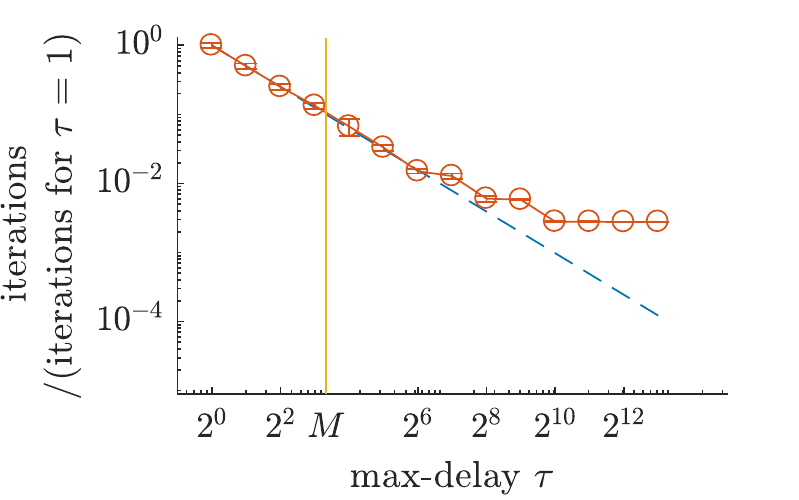}
\hfill
\includegraphics[width=0.32\linewidth]{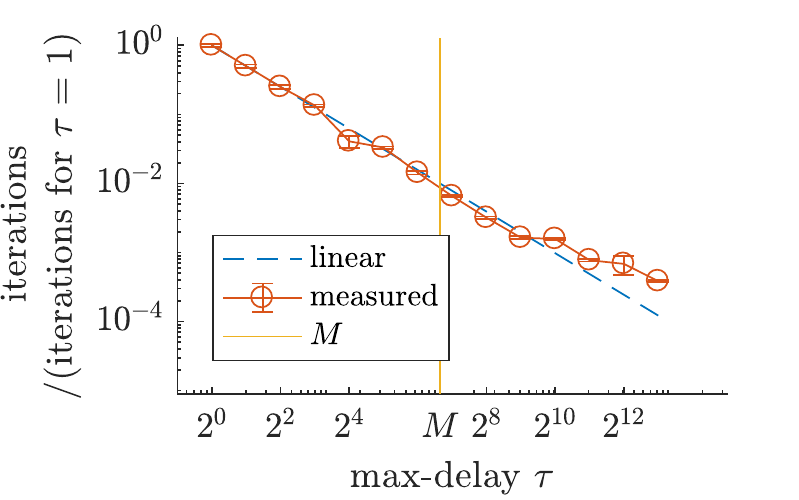}
\hfill
\includegraphics[width=0.32\linewidth]{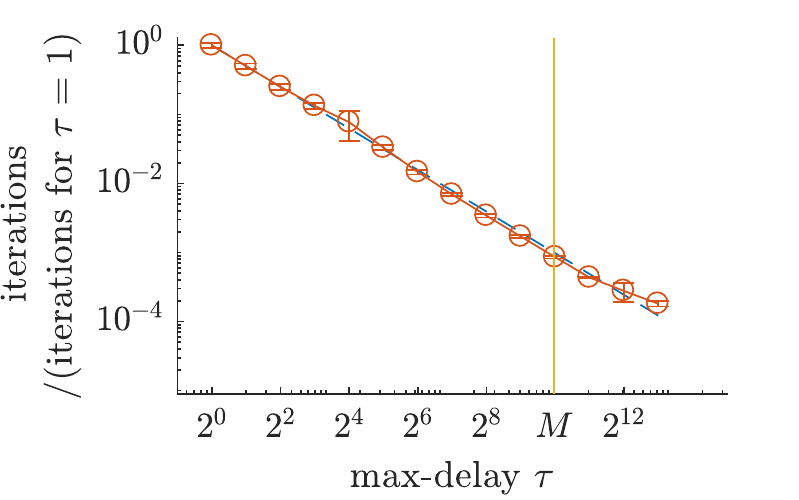}
\hfill\null
\caption{\textbf{Asynchronous SGD (Hogwild!).}  Parallel speedup for delay parameters $\tau \in \{2^0,\dots,2^{14}\}$ and problem instances with $M \in \{10,100,1000\}$ for the same problem setup as in Figure~\ref{fig:new}. See Appendix~\ref{appendix:hogwild} for further details.}
\label{fig:hogwild}
\end{figure*}

\section{Experiments}
In the previous sections we argued theoretically that parallelism up to the critical level $\tau =\cO\bigl(\frac{\sigma^2_\star}{\epsilon} + M \bigr)$ can yield linear speedup in parallel computation time. In this section, we experimentally verify this claim.

In the main paper we focus mainly on mini-batch SGD with varying batch sizes, and provide additional experiments in the appendix for asynchronous versions of SGD (including Hogwild!), with different delay patterns.

\subsection{Scaling on controlled problem instance}\label{ssec:synth}
First, we consider a family of controlled problem instances, corresponding to regularized linear regression problems, where we can control the noise (such as to control the parameter $M$).
We consider the quadratic function $f \colon \R^d \to \R$,\vspace{-2mm}
\begin{align*}
 f(\xx) := \frac{1}{2}\lin{A\xx,\xx} + \frac{\lambda}{2}\norm{\xx}^2\,,
\end{align*}
for $d =20$, $\lambda = 0.2$, and band-diagonal matrix $A$ with $[-\1_{d-1}, 2\1_d,-\1_{d-1}]$ on the diagonals. Without regularization (and without noise) this is a numerically challenging problem for first order methods~\cite[cf.][]{Nesterov2004:book}; with regularization the condition number reduces to approximately $\kappa \leq 19$.
We define stochastic gradients with  $\uu \sim \cN(\0,M \norm{\nabla f(\xx)}^2 \cdot I_d)$ as
\begin{align*}
 \gg(\xx) := \nabla f(\xx) + \uu\,,
\end{align*}
thus it holds $\E\norm{\gg(\xx) -\nabla f(\xx)}^2 \leq M \norm{\nabla f(\xx)}^2$.

In Figure~\ref{fig:new}
we depict the number of iterations required by mini-batch SGD to reach the target accuracy $\frac{1}{d}\norm{\xx_t}\leq 0.1$, for $\xx_0 = 10\cdot \1_d$, and for the best choice of stepsized (tuned over a logarithmic grid $\gamma\in \frac{1.1}{1+M} \cdot \{2^{-1},\dots,2^{-20}\}$, optimal values are always different from the largest or smallest value in this grid). We observe that the value $M$ provides a lower bound on the level of parallelism that enjoys linear speedup, tracking the speedup saturation in the right order of magnitude, but slightly too conservative on this family of problem instances. 
Similar observations also hold for asynchronous methods, as displayed in Figure~\ref{fig:hogwild} (for more details we refer to Appendix~\ref{appendix:hogwild} and~\ref{appendix:delayedsgd}).

\subsection{Measuring the critical batch size on deep learning tasks}
\label{sec:dl}

We now %
aim to understand whether our proposed critical batch sizes correlate with speedup saturation observed in practice~\cite{Shallue2019:batchsize}.
We consider image classification for the CIFAR-10~\citep{krizhevsky2009learning} dataset with
ResNet-8 and ResNet-18~\citep{he2016deep} architectures. 
We train these models for 200 epochs using mini-batch SGD with a momentum of $0.9$, and $5 \cdot 10^{-4}$ weight decay.   

As a heuristic measure of the critical batch size, we are tracking the evolution of the estimator
\begin{align}
 \hat b(\xx) := 1+ \frac{ \E \norm{\gg(\xx)-\nabla f(\xx)}^2  }{ \norm{\nabla f(\xx)}^2 + \hat \epsilon_T }\,, \label{eq:Mnew}
\end{align}
where $\hat \epsilon_T$ is an estimate of the gradient $\norm{\nabla f(\xx_{T})}^2$ at the end of training, measured by taking the average over the last 10 epochs: $\hat \epsilon_T = \frac{1}{10} \sum_{i=0}^9 \norm{\nabla f(\xx_{T-i \cdot n})}^2$, where $n$ is the training data set size, and $T$ the final iteration index.\footnote{
We show in Appendix~\ref{sec:hatb} that \vspace*{-4mm}
\begin{align*}
 \hat b(\xx) \leq 1 + \frac{\sigma_\star^2}{\max\{\norm{\nabla f(\xx)}^2,\hat \epsilon_T \}} + M \leq 4 \sup_{\xx \in \R^d} \hat b(\xx)\,,
\end{align*}
\par\vspace*{-4mm}\noindent
and  hence $\hat b(\xx)$ can be seen as a local estimate of $b_{\rm crit}$ for $\epsilon= \max\{\norm{\nabla f(\xx)}^2,\hat \epsilon_T \}$.
}

In Figure~\ref{fig:resnet18-no-lr-decay} we show the evolution of $\hat b(\xx)$ when training on ResNet-18 for different batch sizes. Additionally, we investigate the effect of training with and without batch normalization with $256$ batch size. We use $0.01$ step size and do not apply any learning rate decay. The estimated value saturates around $5000$, matching with typically observed saturation levels for all batch sizes except for $b = 16$, which is the only batch size not converging to a high accuracy.

We see that batch norm changes the optimization landscape and the training trajectory. Without batch norm, the estimated critical scaling parameter is much lower throughout the training than with batch norm enabled. Note that while we use different batch sizes in training, we compute all the metrics such as $\bxi(\xx)$ and $\nabla f(\xx)$ with batch size $256$.  

\begin{figure}[t]
\includegraphics[width=\linewidth]{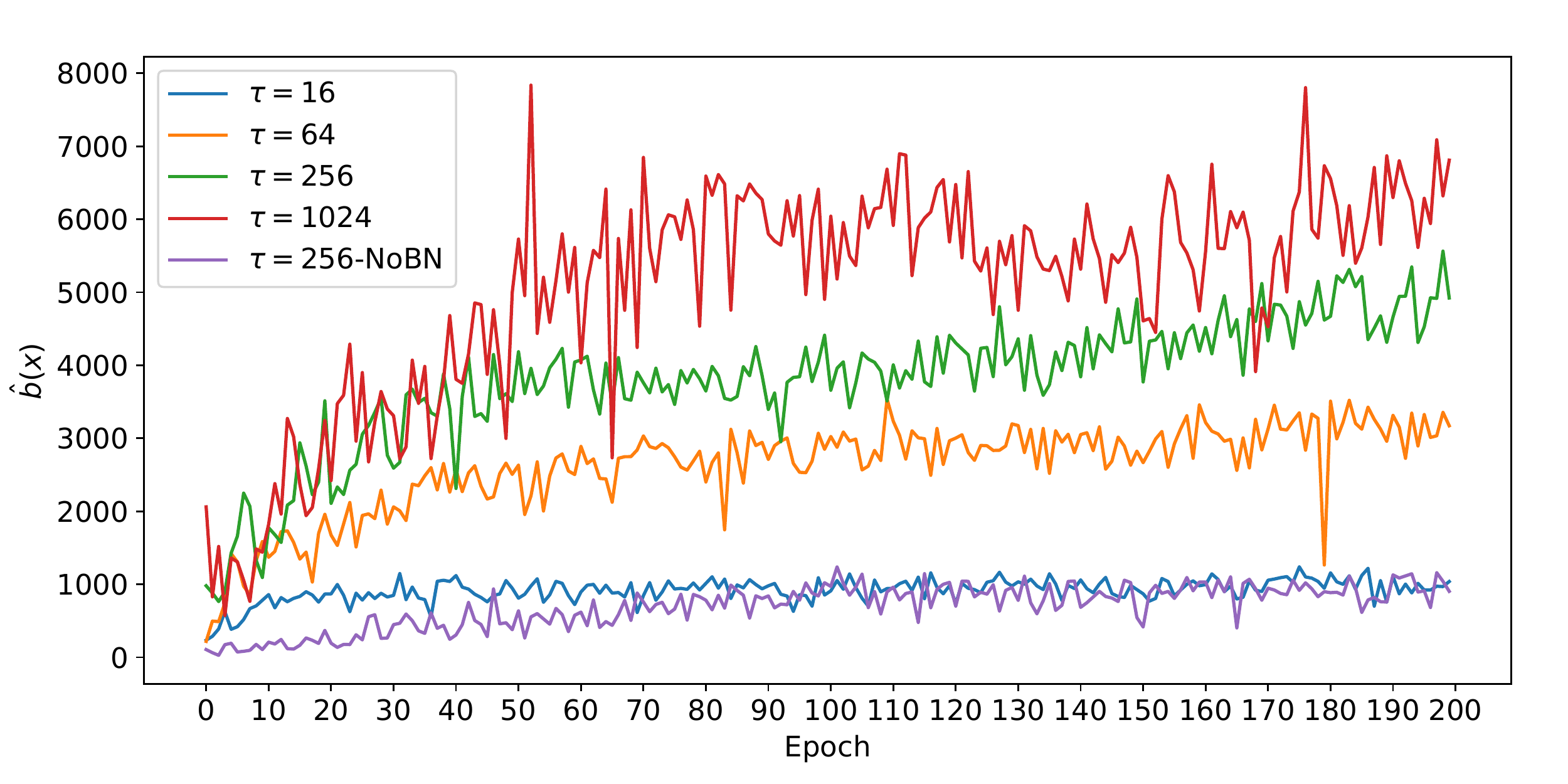}
\caption{Evolution of the critical parameter $\hat b(\xx)$
during training of \textbf{ResNet-18} on CIFAR-10, with batch size $b\in\{16,64,256,1024\}$. 
}
\label{fig:resnet18-no-lr-decay}
\vspace{-4mm}
\end{figure}

\begin{figure}[t]
\includegraphics[width=\linewidth]{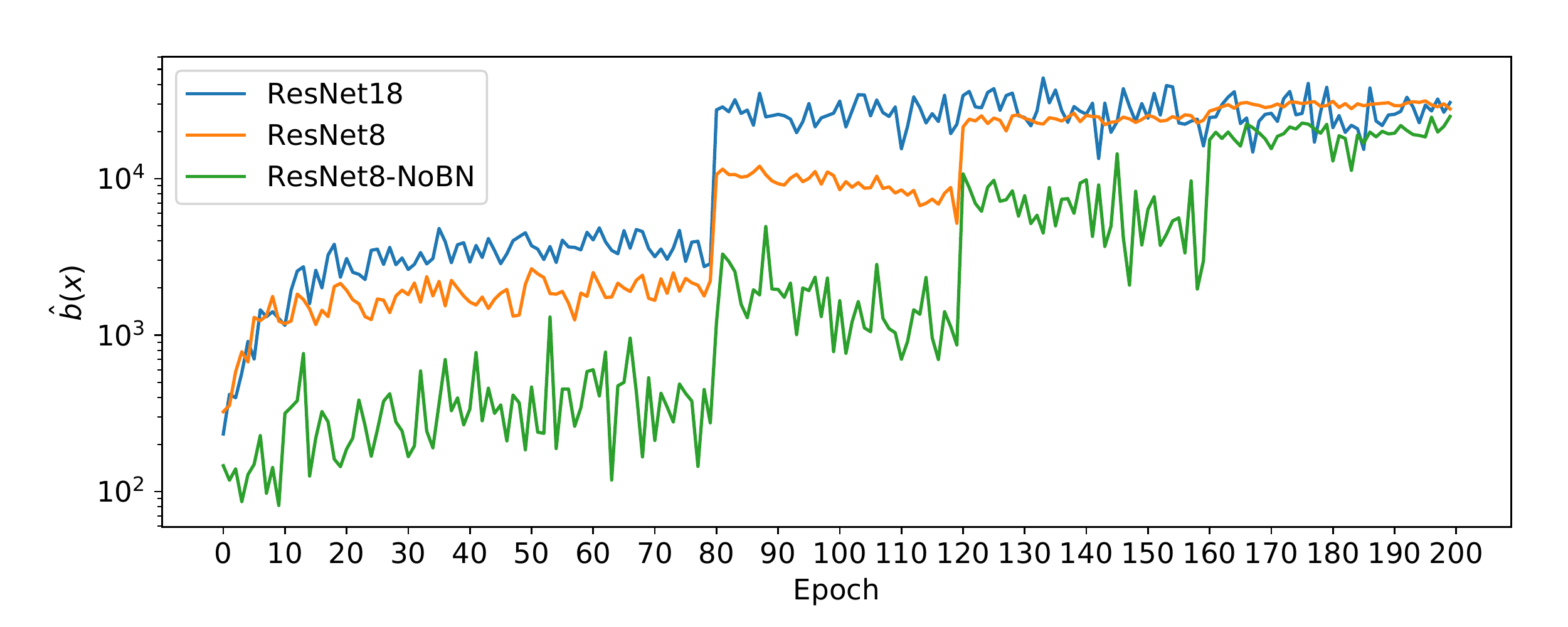}
\caption{Evolution of the critical parameter $\hat b(\xx)$ during training \textbf{ResNet-8} and \textbf{ResNet-18} on CIFAR-10 with batch size $b=256$.}
\label{fig:resnet-lr-decay}
\end{figure}

Next, we use a similar setup as in~\cite{Shallue2019:batchsize} and measure the evolution of $\hat b(\xx)$ when using learning rate decay to train ResNet-18 and ResNet-8 with and without batch norm, depicted in Figure~\ref{fig:resnet-lr-decay}. We use $0.1$ initial step size and decay by a factor of $0.1$ at epochs $80$, $120$, and $160$. The estimated value of $\hat b (\xx)$ increases after each learning rate decay. A decay in learning rate results in a sudden decrease in the gradient's norm, lowering the target error for which $\hat b(\xx)$ estimates the critical batch size. The change in $\hat b(\xx)$ may therefore be justified by the sudden decrease in the target error.

\subsection{Large Batch Speedup Analysis}
In the previous subsection we studied the evolution of the estimator $\hat b(\xx)$ over training. Since $\hat b(\xx)$ provides a lower bound for $b_{\rm crit}$, We obtain a estimator of the critical batch size by taking the maximal observed value until reaching a certain target accuracy, $\hat b_{\rm crit}=\max_{i} \hat b(\xx_i)$.  We now investigate, how this estimate correlates with the speedup saturation observed in practice.

We train ResNet-8 without Batch Normalization with different batch sizes and separately tuned the step size for each batch size. For each batch size, we trained the network until reaching 70\% test accuracy \cite[as in][]{Shallue2019:batchsize}. We plot the number of
 steps (iterations)
 for each batch size overlayed with the value we estimated for $\hat b_{\rm crit}$ in Figure~\ref{fig:speedup-m}. Our result matches with the previous findings of \cite{Shallue2019:batchsize}.  We repeat the same procedure for ResNet-18 with Batch Normalization and train until reaching 80\% test accuracy. To save computational costs, we only use a subset of batch sizes for estimation of $\hat b(\xx)$. The results are depicted in Figure~\ref{fig:speedup-m-resnet18}. 
\begin{figure}[t]
\includegraphics[width=\linewidth]{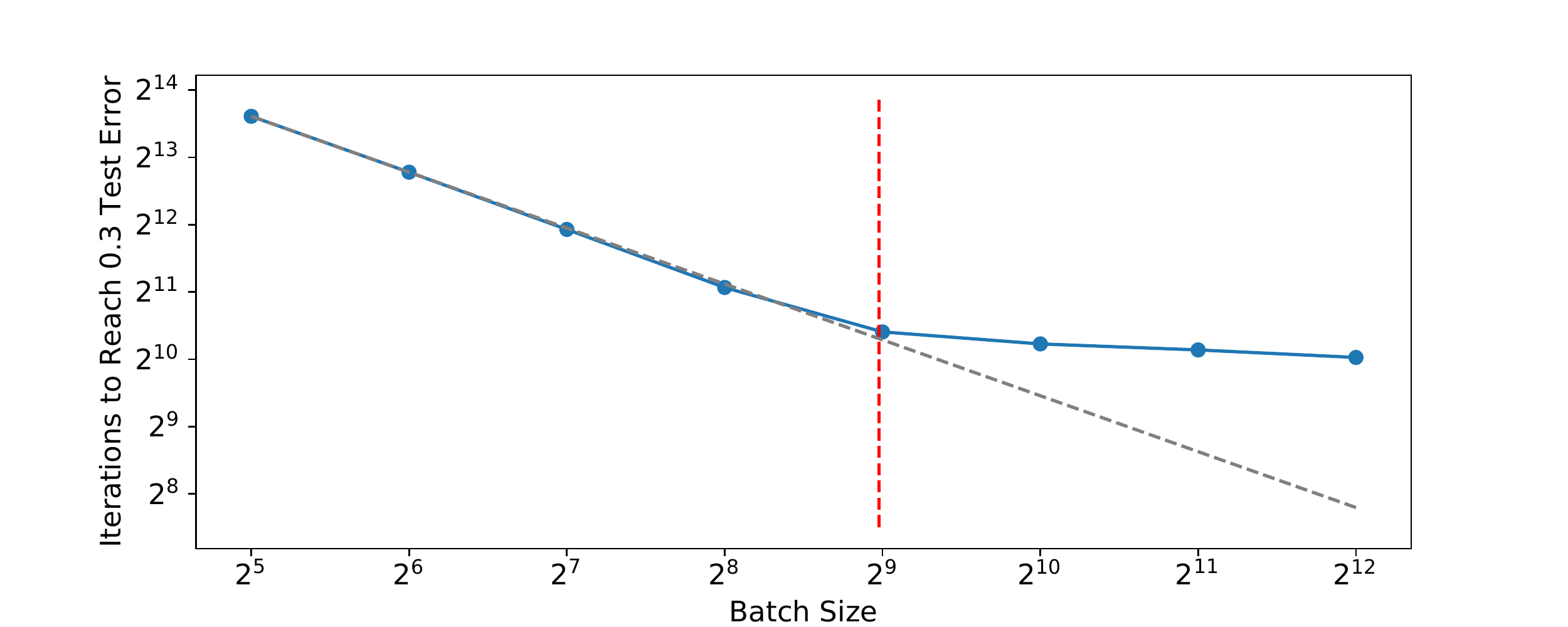}
\caption{\textbf{Linear speedup up to batch size $\hat b_{\rm crit}$.} The number of iterations to reach 0.3 test error with \textbf{ResNet-8} on CIFAR-10 without batch normalization for batch size $b\in\{2^5,\dots,2^{12}\}$.  The red line shows the estimated $\hat b_{\rm crit}$.}
\label{fig:speedup-m}
\end{figure}

\begin{figure}[t]
\includegraphics[width=\linewidth]{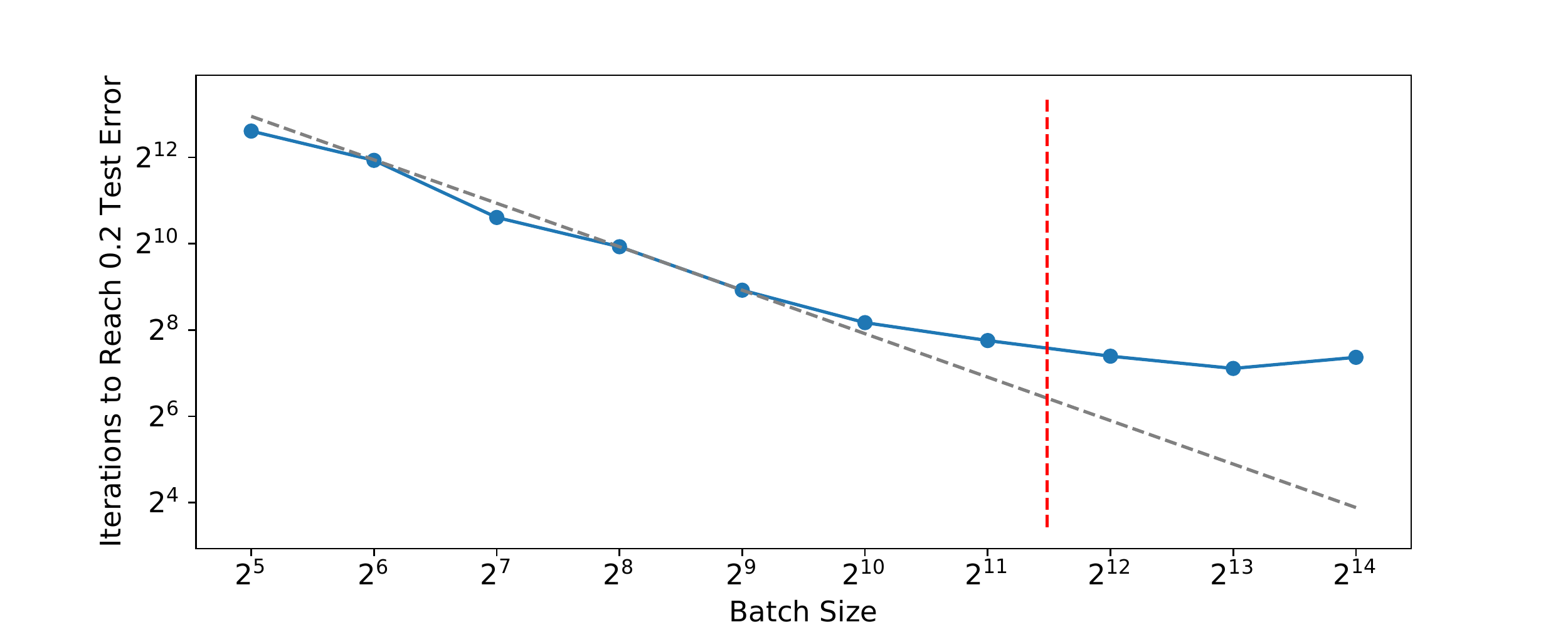}
\caption{The number of iterations to reach 0.2 test error with \textbf{ResNet-18} on CIFAR-10 without batch normalization for batch size $b\in\{2^5,\dots,2^{14}\}$. The red line shows the estimated $\hat b_{\rm crit}$.}
\label{fig:speedup-m-resnet18}
\vspace{-2mm}
\end{figure}
\section{Conclusion}
We introduced a universal parameter that measures the critical level of parallelism (such as e.g.\ batch size, or delays) of stochastic optimization problems that allows for near-linear parallel speedup efficiency. Our notion extends and refines prior notions that could not explain speedup saturation beyond a \emph{constant} critical batch size, closely matching empirical observations. %
Our measurements on deep learning tasks confirm that our proposed metric provides a meaningful estimate also on deep learning tasks.
As future directions we leave it to extend the theory to refined notions of \emph{locally} critical batch sizes (as e.g.\ considered in Section~\ref{sec:dl}, or in~\citealt{Yin2018:diversity,Jain2018:parallelSGD})
and to study theoretically the generalization gap in large batch training~\cite{keskar2016large}.

\small
\bibliographystyle{myplainnat}

\bibliography{papers_hogwild}
\normalsize

\appendix
\onecolumn

\section{On the proof of Theorem~\ref{thm:main}}
\label{sec:thm}
For the proof of Theorem~\ref{thm:main} we resort to techniques and results that have been developed in prior works and that can easily be adapted to the setting considered here, in particular~\cite[Theorem 15]{StichK19delays}. This theorem addresses the particular case when the gradients in Algorithm~\ref{alg:hogwild} are all delayed by a delay of \emph{exactly} $\tau$ (a.k.a. \emph{delayed SGD}). In contrast, we consider here the setting were coordinates of the gradients can be delayed independently, delays do not follow a particular order and reading of the variable $\xx$ from the memory can be inconsistent. However, the proof in~\cite{StichK19delays} can easily be adapted to our more general setting \citep[as also observed in][]{Elalamy2020:semesterproject} and we do not claim much novelty here---except of explicitly stating this generalization.

\subsection{Proof Overview}
The proof in~\cite{StichK19delays} for the convex case follows
by refining the perturbed iterated framework, developed in~\cite{Mania2017:perturbed} and extended in~\cite{Leblond2018:async}. 
A key ingredient in the proof is to consider 
a (virtual, ghost) sequence
\begin{align*}
 \tilde \xx_{t+1} := \tilde \xx_t - \gamma_t \gg_t %
\end{align*}
with $\gg_t=\gg_t(\xx_t)$. In the following we resort---for the ease of presentation---to constant step sizes $\gamma_t \equiv \gamma$. 

For instance for convex functions, it can be shown~\cite[Lemma 7]{StichK19delays} that the perturbed iterates satisfy
\begin{align}
 \E{\norm{\tilde \xx_{t+1} - \xx^\star}^2} &\leq \left(1-\frac{\mu \gamma}{2}\right) \E{\norm{\tilde \xx_t - \xx^\star}^2}  - \frac{\gamma}{2} (\E{f(\xx_t)}-f^\star) + \gamma^2 \sigma^2 + 3 L \gamma \underbrace{\E{\norm{\xx_t - \tilde \xx_t}^2}}_{=:R_t}\,, \label{eq:dconvex}
\end{align}
and for non-convex functions~\cite[Lemma 8]{StichK19delays}:
\begin{align}
  \E{f(\tilde \xx_{t+1})} \leq \E{f(\xx_t)} - \frac{\gamma}{4}\E{\norm{\nabla f(\xx_t)}^2} + \frac{\gamma^2 L}{2} \sigma^2 + \frac{\gamma L}{2} \underbrace{\E{\norm{\xx_t - \tilde \xx_t}^2}}_{=:R_t}\,. \label{eq:nonconvex}
\end{align}

\subsection{Bound on $R_t$ in~\citep{StichK19delays}}
\citet{StichK19delays} analyze the convergence of a \emph{delayed gradient method}, as introduced in \cite{Arjevani2018:delayed} and provide an upper bound for the value of $R_t$.

\begin{lemma}[{\citealt[Lemma 10]{StichK19delays}}]
Let $\gamma \leq  \frac{1}{10 L(\tau+M)}$ and $\xx_t$ defined as $\xx_{t+1} := \xx_t - \gamma \gg_{t-\tau}$ for $t \geq \tau$, and $\xx_t = \xx_0$ for $t \in \{0,\dots, \tau-1\}$ (\emph{delayed SGD}). Then
\begin{align}
 R_t := \Eb{\norm{\xx_t - \tilde \xx_t}^2 } \leq \frac{1}{30 L^2 \tau}  \sum_{\mathclap{k=(t-\tau)_+}}^{t-1}  \E \norm{\nabla f(\xx_k)}^2  + \frac{2}{3L }\gamma \sigma^2 =: \Theta_{\rm SK} \,. \label{eq:dsgd_bound1}
\end{align}
\end{lemma}

\subsection{Bound on $R_t$ under $\tau$ bounded parallelism}
We now switch to our setting and derive a similar bound on $R_t$ that holds for the more general class of algorithms considered in Theorem~\ref{thm:main}.

\begin{lemma}\label{lemma:error}
It holds
\begin{align*}
 R_t =\E{\norm{\xx_t - \tilde \xx_t}^2} & \leq  2\gamma^2 (\tau +M) \sum_{\mathclap{k=(t-\tau)_+}}^{t-1}  \E \norm{\nabla f(\xx_k)}^2  + 2\gamma^2 \tau \sigma^2 \,.
\end{align*}
and in particular for $\gamma \leq \gamma_{\rm crit}=\frac{1}{10L(M+\tau)}$
\begin{align}
 R_t \leq \frac{1}{50L^2 \tau} \sum_{\mathclap{k=(t-\tau)_+}}^{t-1}  \E \norm{\nabla f(\xx_k)}^2  + \frac{1}{5L} \gamma \sigma^2 =: \Theta_{\rm SMJ} \,.  \label{eq:dsgd_bound2}
\end{align}
\end{lemma}
We observe that our bound provided in~\eqref{eq:dsgd_bound2} is smaller than the bound provided in~\eqref{eq:dsgd_bound1}, i.e., $\Theta_{\rm SMJ} \leq \Theta_{\rm SK}$. Therefore, the proof of Theorem~\ref{thm:main} now follows from~\cite[Theorem 16]{StichK19delays} (that only relies on the weaker bound $\Theta_{\rm SK}$).

\begin{proof}[Proof of Lemma~\ref{lemma:error}]
First, we observe that by definition of $\xx_t$ and $\tilde \xx_t$ and the maximal overlap $\tau$, we can write
\begin{align}
 \norm{\xx_t - \tilde \xx_t}^2 := \norm{ \gamma \sum_{k < t}  (\mJ^t_k-\mI_d) \gg_k  }^2 = \norm{ \gamma \sum_{k=(t-\tau)_+}^{t-1}   (\mJ^t_k-\mI_d) \gg_k }^2\,,
\end{align}
where $\gg_k := \nabla f(\xx_k) + \bxi_k$ for zero-mean noise terms. Therefore
\begin{align*}
\E \norm{\xx_t - \tilde \xx_t}^2 &\stackrel{\text{\ding{172}}}{\leq} 2\gamma^2 \left( \E \norm{ \sum_{k=(t-\tau)_+}^{t-1} (\mJ^t_k-\mI_d)  \nabla f(\xx_k)}^2 + \E \norm{ \sum_{k=(t-\tau)_+}^{t-1} (\mJ^t_k-\mI_d) \bxi_k}^2 \right) \\
&\stackrel{\text{\ding{173}}}{\leq} 2\gamma^2 \left( \tau \sum_{k=(t-\tau)_+}^{t-1} \E \norm{(\mJ^t_k-\mI_d)  \nabla f(\xx_k)}^2 + \sum_{k=(t-\tau)_+}^{t-1} \E \norm{(\mJ^t_k-\mI_d) \bxi_k}^2 \right) \\
&\stackrel{\text{\ding{174}}}{\leq} 2\gamma^2 \left( \tau \sum_{k=(t-\tau)_+}^{t-1} \E \norm{ \nabla f(\xx_k)}^2 + \sum_{k=(t-\tau)_+}^{t-1} \E \norm{\bxi_k}^2 \right) \\
&\stackrel{\text{\ding{175}}}{\leq} 2\gamma^2 \left( (\tau+M) \sum_{k=(t-\tau)_+}^{t-1} \E \norm{ \nabla f(\xx_k)}^2 + \tau \sigma^2 \right)\,,
\end{align*}
where we used
\ding{172} $\norm{\aa + \bb}^2 \leq 2\norm{\aa}^2+ 2\norm{\bb}^2$,
\ding{173} $\norm{\sum_{i=1}^\tau \aa_i}^2 \leq \tau \sum_{i=1}^\tau \norm{\aa_i}^2$, and $\E \norm{\sum_{i=1}^\tau \bxi_i}^2 = \sum_{i=1}^\tau \E \norm{\bxi_k}^2$,
\ding{174} $\norm{(\mJ^t_k-\mI_d)\nabla f(\xx_k)}^2 \leq \norm{\mJ^t_k-\mI_d}^2 \norm{\gg_k}^2 \leq \norm{\nabla f(\xx_k)}^2$,
\ding{175} $\E \norm{\bxi_k}^2 \leq M\norm{\nabla f(\xx_k)}^2 + \sigma^2$.
\end{proof}

\subsection{Concluding the proof}
As mentioned above, the proof now follows directly from~\cite[Theorem 16]{StichK19delays}. To make this paper more self-contained, we illustrate the remaining steps for the case of non-convex functions.

For the non-convex case, equation~\eqref{eq:nonconvex} gives us the progress of one step. Using notation $r_t := 4\E (f(\tilde \xx_{t}) - f^\star)$, $s_t := \E \norm{\nabla f(\xx_t)}^2$, and $c = 4L \sigma^2$ we have
\begin{align*}
    \frac{1}{4 T} \sum_{t=0}^T  s_t &\stackrel{\eqref{eq:nonconvex}}{\leq} \frac{1}{T}\sum_{t=0}^T \left(\frac{r_t}{4\gamma_t} - \frac{r_{t+1}}{4\gamma} + \frac{\gamma c}{8}\right) + \frac{L^2}{2T} \sum_{t=0}^T \E \norm{\xx_t - \tilde \xx_t}^2\\
   &\stackrel{\substack{(\Theta_{\rm SMJ} \leq \Theta_{\rm SK})\\\eqref{eq:dsgd_bound1}}}{\leq} \frac{1}{T}\sum_{t=0}^T \left(\frac{r_t}{4\gamma} - \frac{r_{t+1}}{4\gamma} + \frac{\gamma c}{8}\right) + \frac{L^2}{2T} \sum_{t=0}^T\left( \frac{1}{15 L^2 } s_t + \frac{\gamma c}{4 L^2}\right)\,.
\end{align*}
The above equation can be simplified as: 
\[
    \frac{1}{5 T} \sum_{t=0}^T s_t \leq \frac{1}{T}\sum_{t=0}^T \left(\frac{r_t}{4\gamma} - \frac{r_{t+1}}{4\gamma} + \frac{\gamma c}{4}\right) \leq \frac{r_0}{4 \gamma T } + \frac{\gamma c}{4} \,.
\]
Now, the claimed bound follows by choosing the optimal stepsize $\gamma \leq \gamma_{\rm crit}$ that minimizes the right hand side. For this refer e.g. to \cite[Lemma 14]{StichK19delays} or \cite{Arjevani2018:delayed}.

The proof for the convex cases start from the one step progress provided in~\eqref{eq:dconvex} instead, and proceed similarly.

\newpage
\section{Additional numerical experiments}
In this section we report additional empirical results for the setting considered in Section~\ref{ssec:synth}. We consider three algorithms with the same level of parallelism: mini-batch SGD as considered in the main text, and two implementations of SGD with delayed updates.

\subsection{On the estimator $\hat b(\xx)$}
\label{sec:hatb}

Note that
\begin{align}
 1+ \frac{\sigma_\star^2}{ \max\{ \hat \epsilon_T, \norm{\nabla f(\xx) }^2\}} + M \geq 1+ \frac{M\norm{\nabla f(\xx)}^2 + \sigma_\star^2}{\norm{\nabla f(\xx)}^2 + \hat \epsilon_T} \stackrel{\eqref{eq:noise}}{\geq}  \hat b(\xx)\,.
\end{align}
Moreover, for $\tilde \epsilon := \max\{ \hat \epsilon_T, \norm{\nabla f(\xx) }^2\}$,
\begin{align}
1 + \frac{\sigma_\star^2}{\tilde \epsilon } + M &\leq 1 + \sup_{\norm{\nabla f(\xx)}^2 \leq  \tilde \epsilon } \frac{\E \norm{\bxi(\xx)}^2}{ \tilde \epsilon } + \sup_{\norm{\nabla f(\xx)}^2 \geq \tilde \epsilon } \frac{\E \norm{\bxi(\xx)}^2}{\norm{\nabla f(\xx)}^2} \\
&\leq 1 + 2 \sup_{\xx \in \R^d } \frac{\E \norm{\bxi(\xx)}^2}{ \tilde \epsilon + \norm{\nabla f(\xx)}^2} + 2 \sup_{\xx \in \R^d } \frac{\E \norm{\bxi(\xx)}^2}{\tilde \epsilon + \norm{\nabla f(\xx)}^2} \\
 &\leq 4 \sup_{\xx \in \R^d} \hat b(\xx) \,.
\end{align}

This method of measuring the critical batch size might not be too accurate. We use this estimator 
only to show that our theoretical findings match our observations in practice and to show how they can be used
to explain phenomena such as critical batch size and scaling of learning rate. We leave finding a more accurate
and online method for measuring the critical batch size as a possible future work.

\clearpage\newpage
\subsection{Mini-batch SGD}
\label{appendix:minibatch}
We consider standard mini-batch SGD, for batch size $b\geq 1$,
\begin{align*}
 \xx_{t+b} = \xx_t - \frac{\gamma_{\rm bm}}{b} \sum_{i=1}^b \gg^i(\xx_t)\,,
\end{align*}
where $\gg^i(\xx_t)$ for $i \in [b]$ denotes independently sampled stochastic gradients.

\begin{figure}[h!]
\hfill
\includegraphics[width=0.32\linewidth]{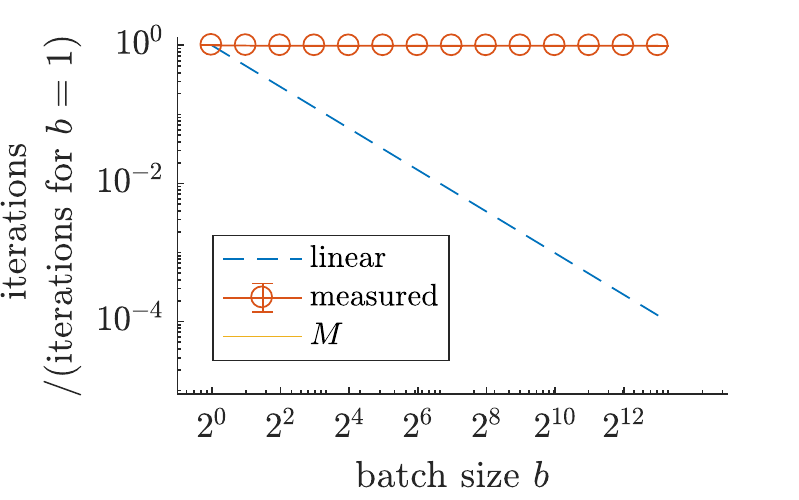}
\hfill
\includegraphics[width=0.32\linewidth]{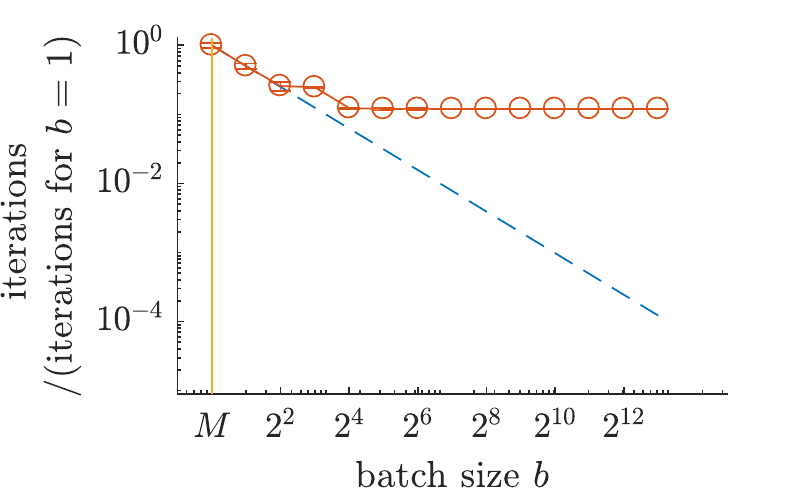}
\hfill
\includegraphics[width=0.32\linewidth]{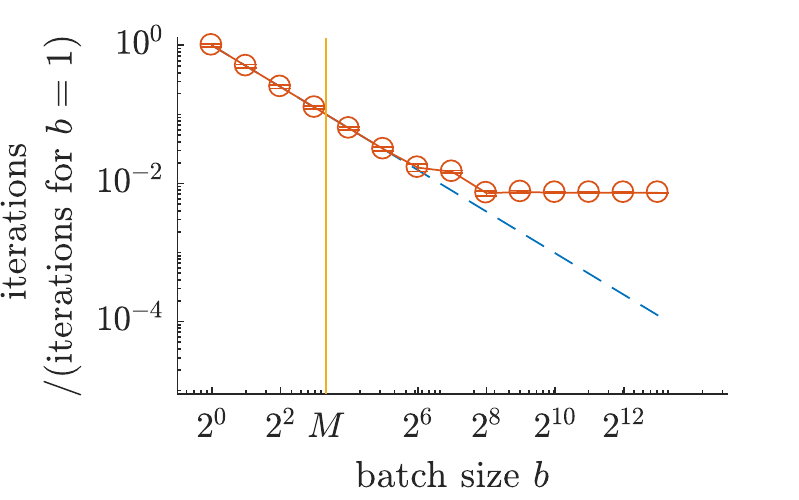}
\hfill\null

\hfill
\includegraphics[width=0.32\linewidth]{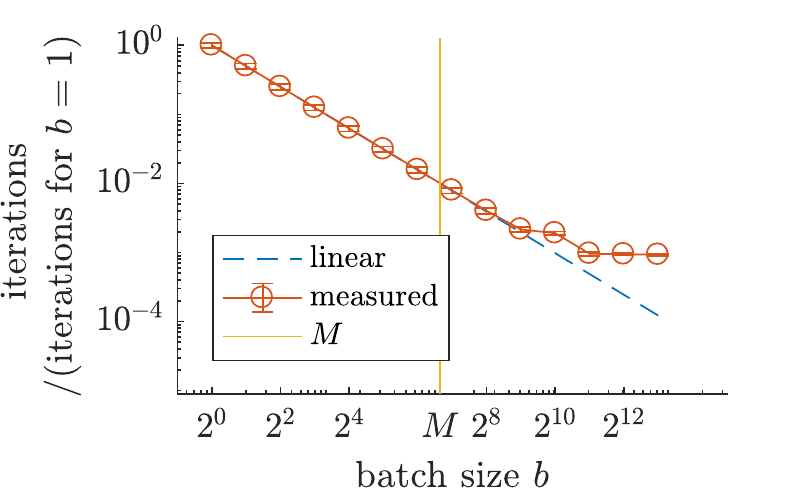}
\hfill
\includegraphics[width=0.32\linewidth]{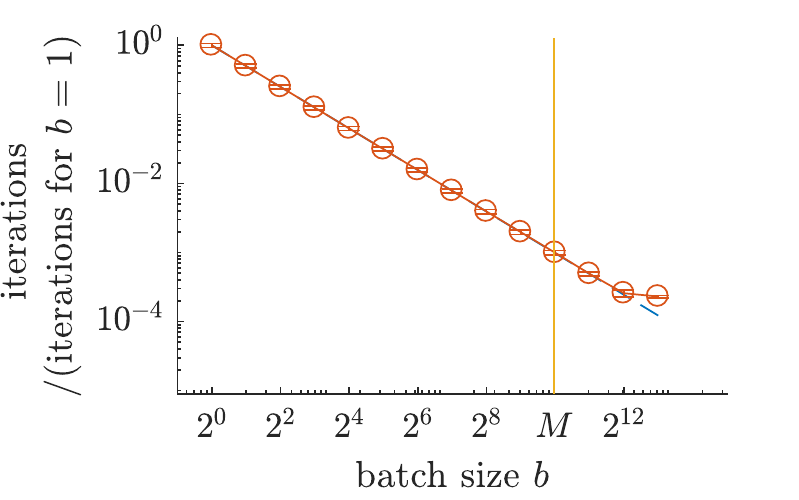}
\hfill
\includegraphics[width=0.32\linewidth]{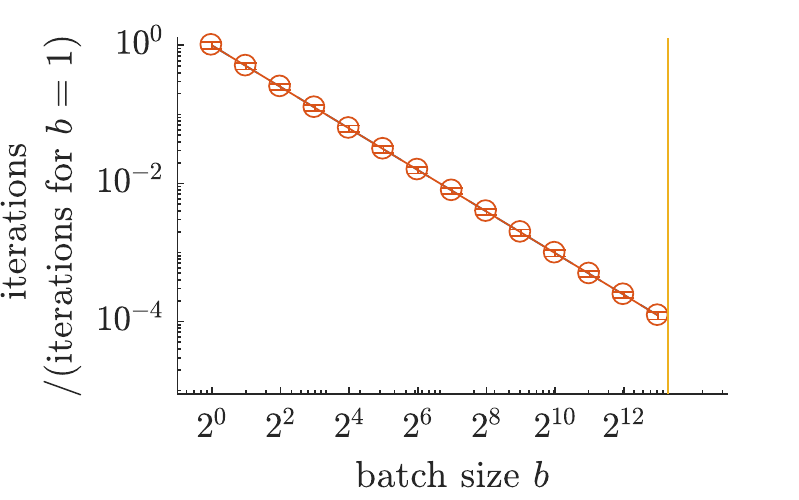}
\hfill\null
\caption{\textbf{Scaling (Mini-batch SGD)}. Parallel speedup for various batch sizes $b \in \{2^0,\dots,2^{14}\}$ and problem instances with $M \in \{0,1,10\}$ (top) and $M \in \{100,1000,10000\}$ (bottom), on the synthetic optimization problem described in Section \ref{ssec:synth}, averaged over three random seeds (depicting mean and $\pm$SD). Plots depict number of iterations (i.e.\ parallel running time $\frac{1}{b} T(b,\epsilon)$), normalized by $T(1,\epsilon)$, required to reach the target accuracy with tuned optimal learning rates.}
\label{fig:scaling_mb}
\end{figure}

\begin{figure}[h!]
\hfill
\includegraphics[width=0.32\linewidth]{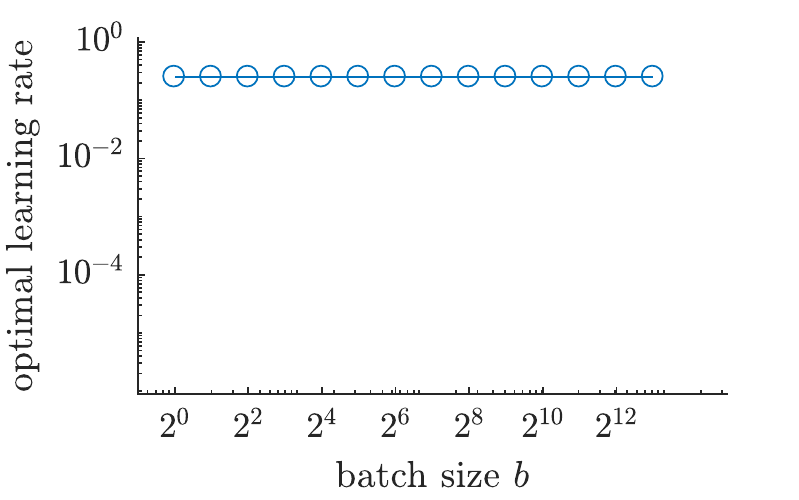}
\hfill
\includegraphics[width=0.32\linewidth]{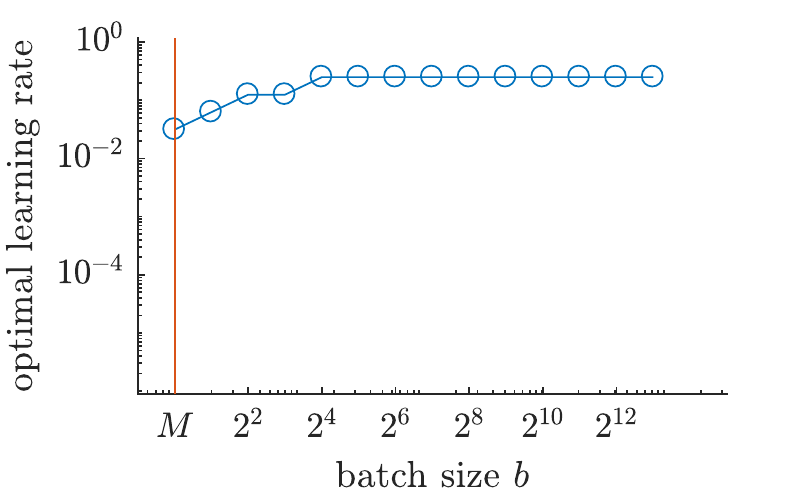}
\hfill
\includegraphics[width=0.32\linewidth]{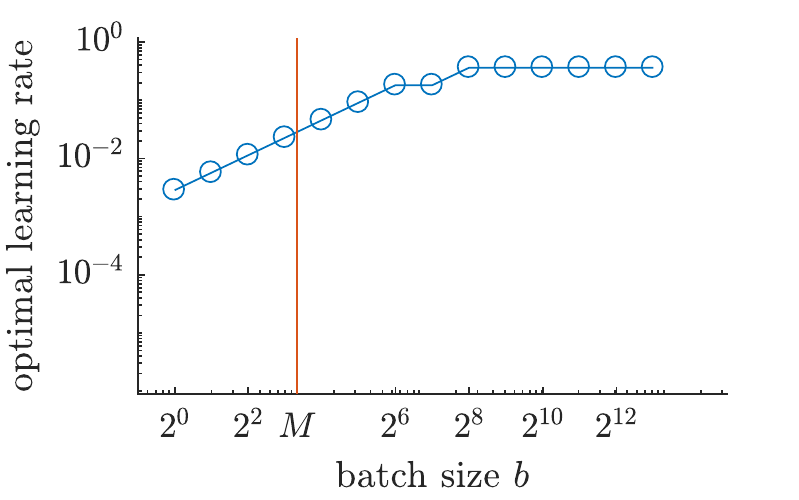}
\hfill\null

\hfill
\includegraphics[width=0.32\linewidth]{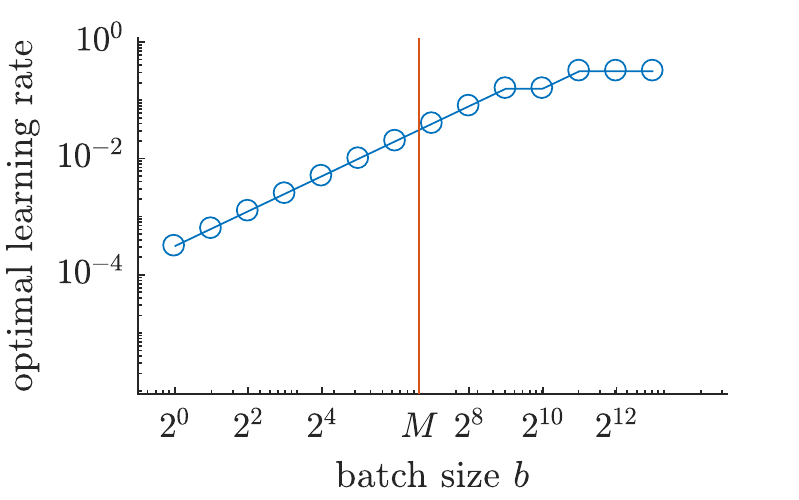}
\hfill
\includegraphics[width=0.32\linewidth]{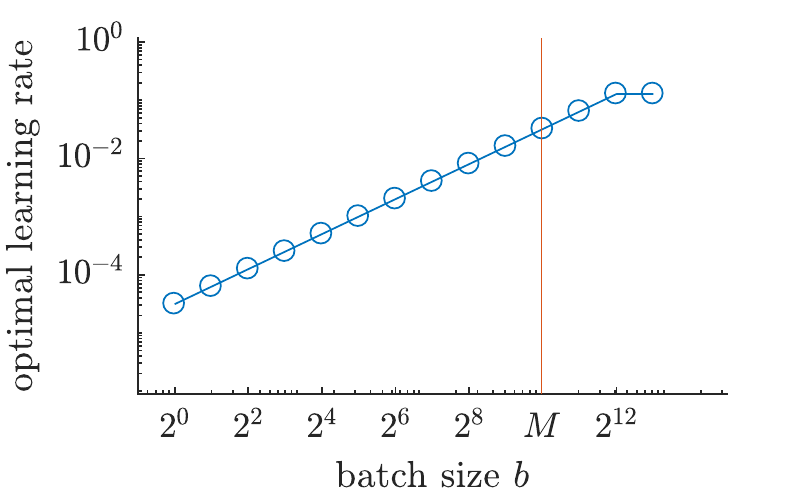}
\hfill
\includegraphics[width=0.32\linewidth]{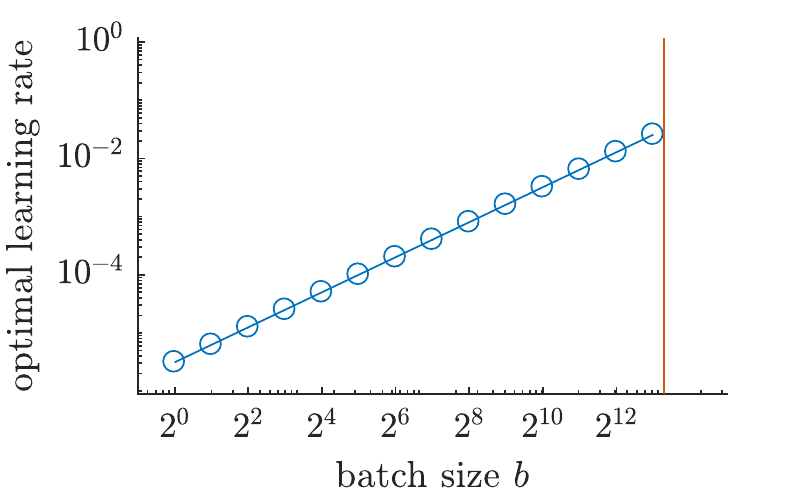}
\hfill\null
\caption{\textbf{Optimal learning rate $\gamma_{\rm mb}$ for mini-batch SGD} for the results reported in Figure~\ref{fig:scaling_mb}.}
\label{fig:lr_mb}
\end{figure}

\clearpage\newpage

\subsection{Delayed SGD (coordinate-wise random delays)}
\label{appendix:hogwild}
In this section we consider SGD with delayed updates. Concretely, we simulate the case where each coordinate $[\gg(\xx_t)]_v$, $v \in [d]$  is delayed for a delay $\tau_{t,v} \sim_{\rm u.a.r.} [\tau]$. This can be seen as a simplistic modeling of Hogwild!~\cite{Recht2011:hogwild}, though in practical settings the delays might be correlated. The update can be written as
\begin{align*}
 \xx_{t+1} = \xx_t - \frac{\gamma_{\rm HW}}{\tau}\sum_{i=t+1-\tau}^{t} \mP_i^t  \gg_i
\end{align*}
where $\gg_t := \gg(\xx_t)$ stochastic gradients (sampled at iteration $t$), 
and $\mP_i^t$ are diagonal matrices with $\sum_{k \geq t}\mP_i^k = \mI_d$, $(\mP_i^t)_{vv} = 1$ if $[\gg_t]_v$ is written at iteration $i \geq t$ and $(\mP_i^t)_{vv} = 0$ otherwise.

\begin{figure}[h!]
\hfill
\includegraphics[width=0.32\linewidth]{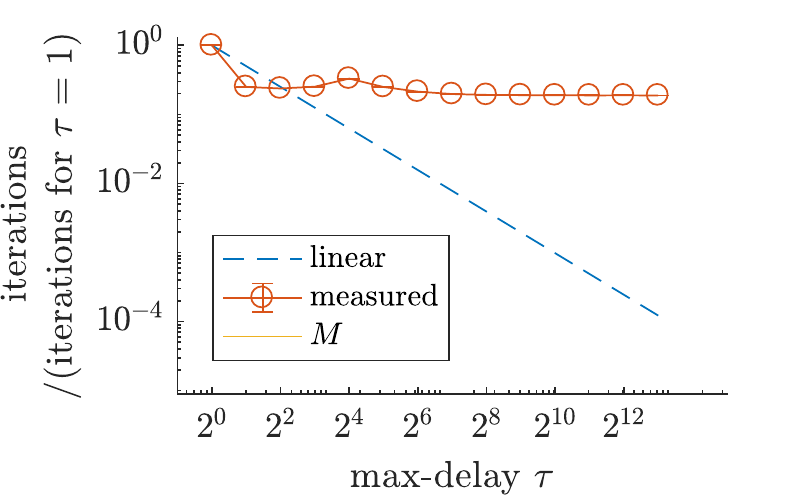}
\hfill
\includegraphics[width=0.32\linewidth]{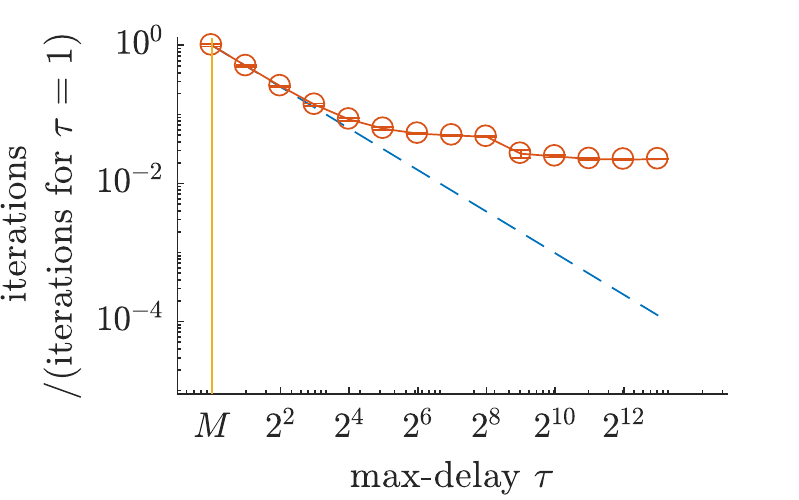}
\hfill
\includegraphics[width=0.32\linewidth]{matlab/final_plots/speedupM=10type=random.pdf}
\hfill\null

\hfill
\includegraphics[width=0.32\linewidth]{matlab/final_plots/speedupM=100type=random.pdf}
\hfill
\includegraphics[width=0.32\linewidth]{matlab/final_plots/speedupM=1000type=random.pdf}
\hfill
\includegraphics[width=0.32\linewidth]{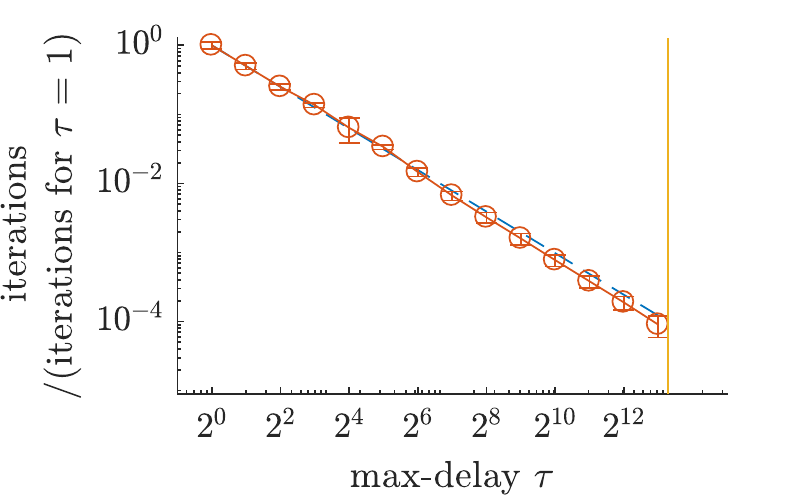}
\hfill\null
\caption{\textbf{Scaling of SGD with coordinate-wise delayed updates (Hogwild!).} Parallel speedup for various batch sizes/delay values $b \in \{2^0,\dots,2^{14}\}$ and problem instances with $M \in \{0,1,10\}$ (top) and $M \in \{100,1000,10000\}$ (bottom), on the synthetic optimization problem described in Section \ref{ssec:synth}, averaged over three random seeds (depicting mean and $\pm$SD). Plots depict number of iterations (i.e.\ parallel running time $\frac{1}{b} T(b,\epsilon)$), normalized by $T(1,\epsilon)$, required to reach the target accuracy with tuned optimal learning rates.}
\label{fig:scaling_random}
\end{figure}

\begin{figure}[h!]
\hfill
\includegraphics[width=0.32\linewidth]{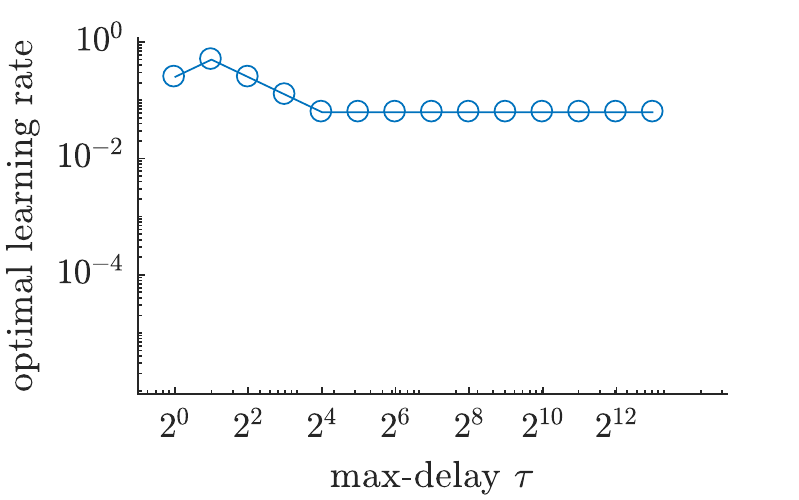}
\hfill
\includegraphics[width=0.32\linewidth]{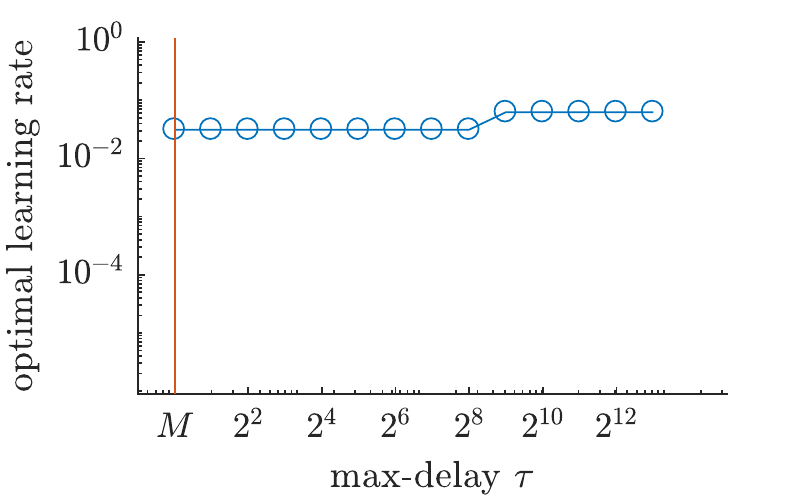}
\hfill
\includegraphics[width=0.32\linewidth]{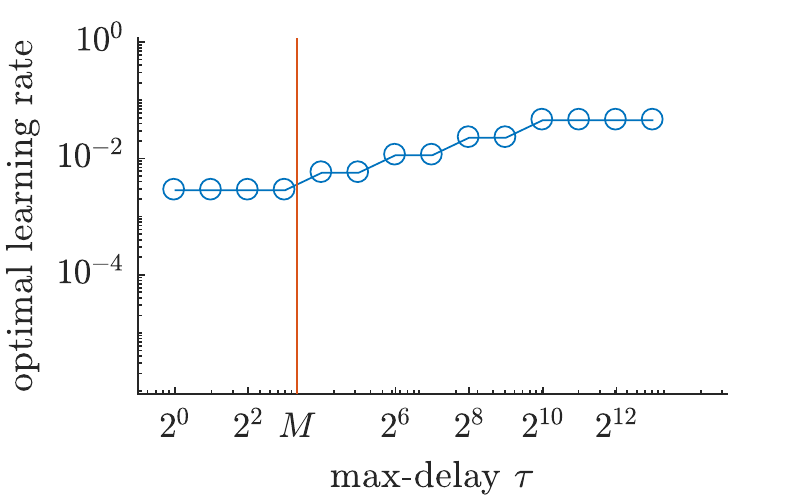}
\hfill\null

\hfill
\includegraphics[width=0.32\linewidth]{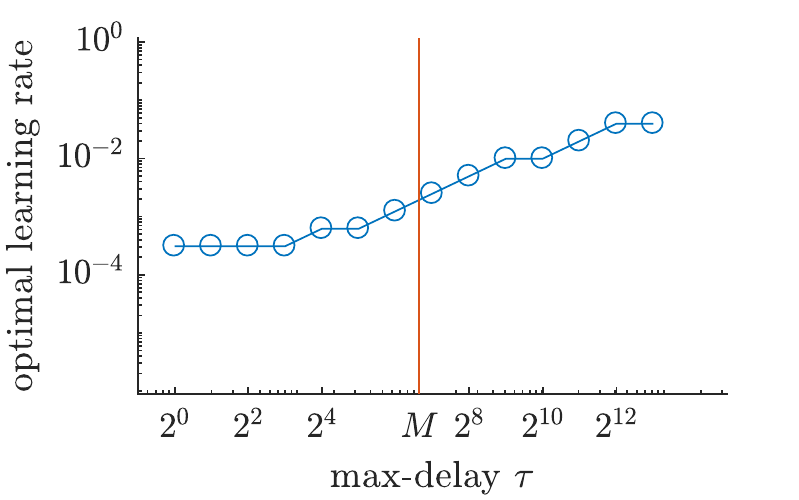}
\hfill
\includegraphics[width=0.32\linewidth]{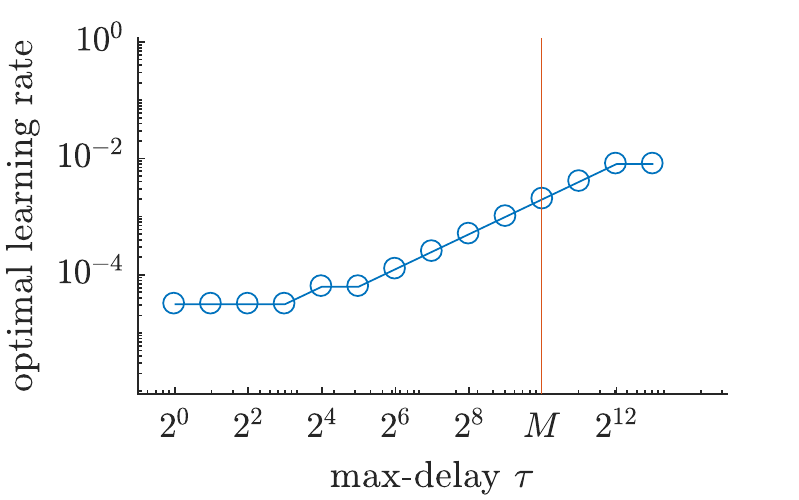}
\hfill
\includegraphics[width=0.32\linewidth]{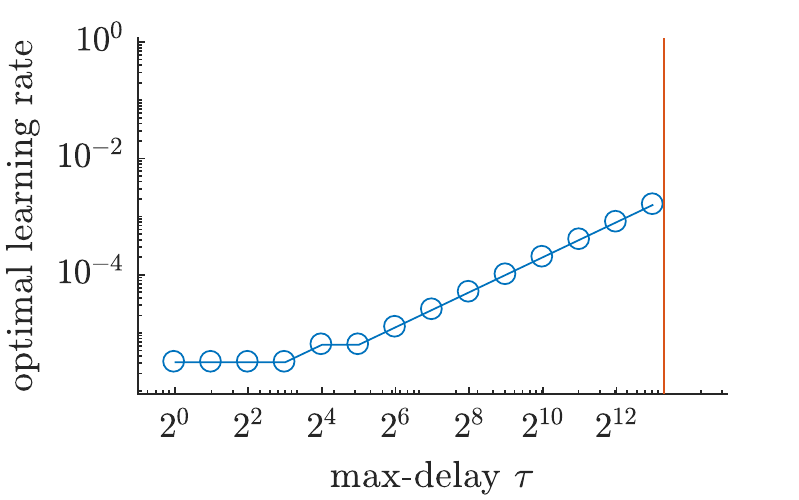}
\hfill\null
\caption{\textbf{Optimal learning rate $\gamma_{\rm HW}$ for Hogwild!} for the results reported in Figure~\ref{fig:scaling_random}.}
\label{fig:lr_random}
\end{figure}

\clearpage\newpage
\subsection{Delayed SGD (worst case delays)}
\label{appendix:delayedsgd}
In this section we consider SGD with delayed updates~\cite{Arjevani2018:delayed}. Concretely, we assume each gradient update is delayed by exactly $\tau$ iterations. For $t \geq \tau$, the update can be written as
\begin{align*}
 \xx_{t+1} = \xx_{t} - \frac{\gamma_{\rm d}}{\tau} \gg(\xx_{t+1-\tau})\,,
\end{align*}
with $\xx_i = \xx_0$ for $i \in [\tau-1]$.
\begin{figure}[h!]
\hfill
\includegraphics[width=0.32\linewidth]{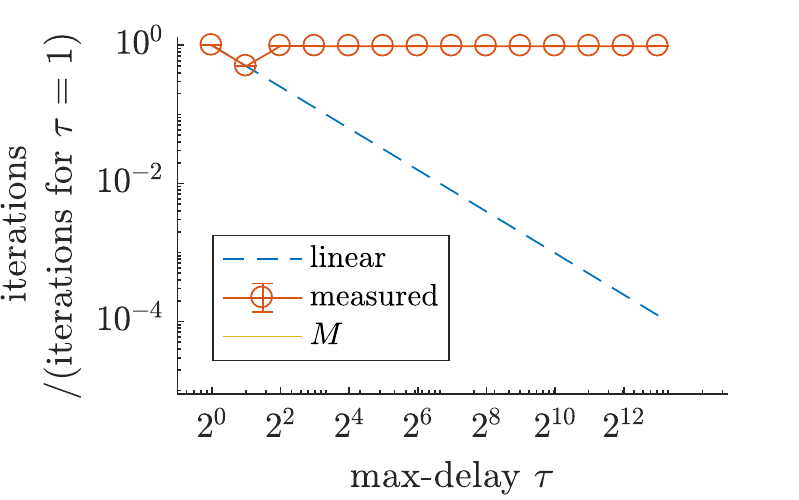}
\hfill
\includegraphics[width=0.32\linewidth]{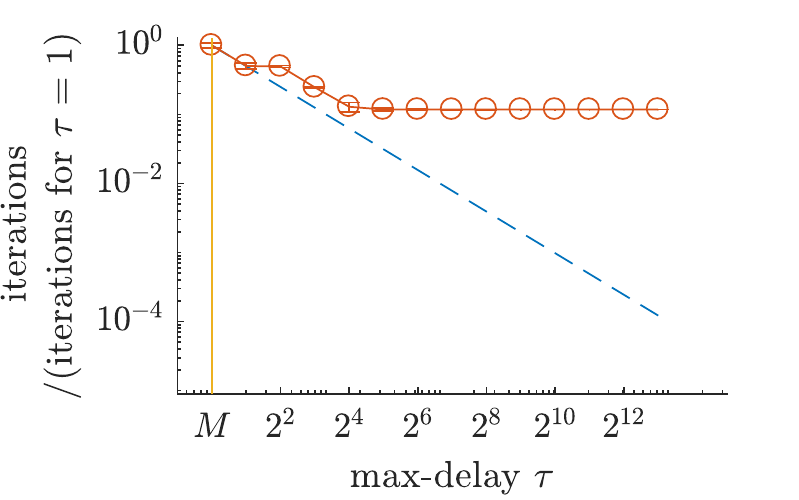}
\hfill
\includegraphics[width=0.32\linewidth]{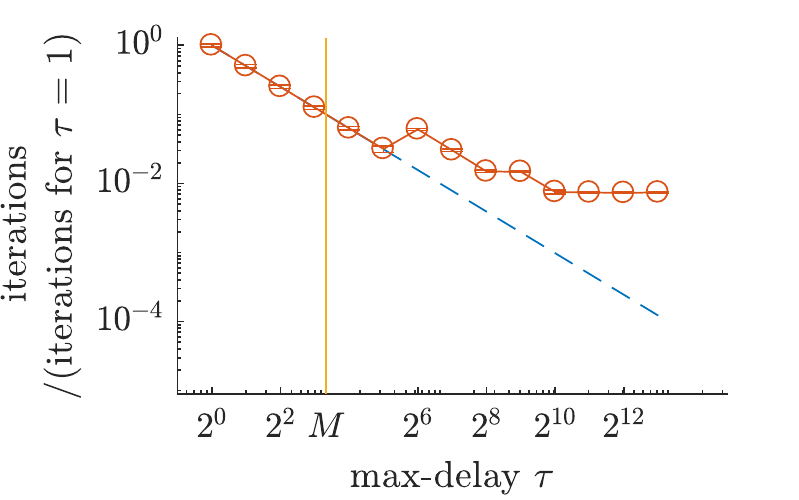}
\hfill\null

\hfill
\includegraphics[width=0.32\linewidth]{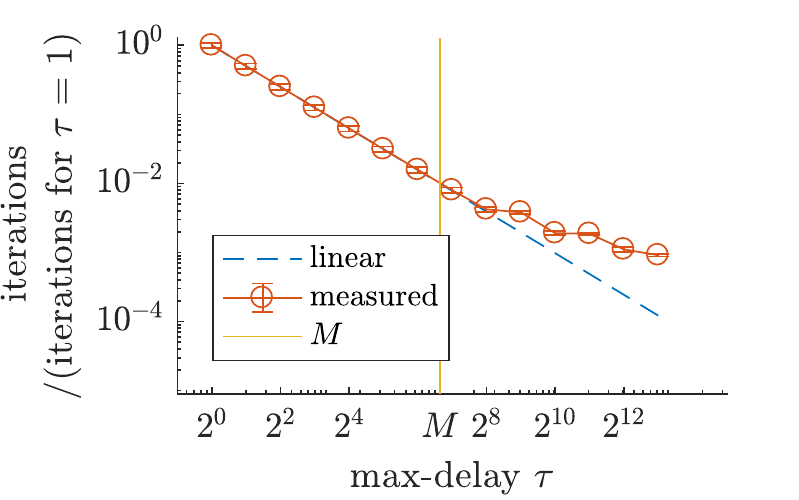}
\hfill
\includegraphics[width=0.32\linewidth]{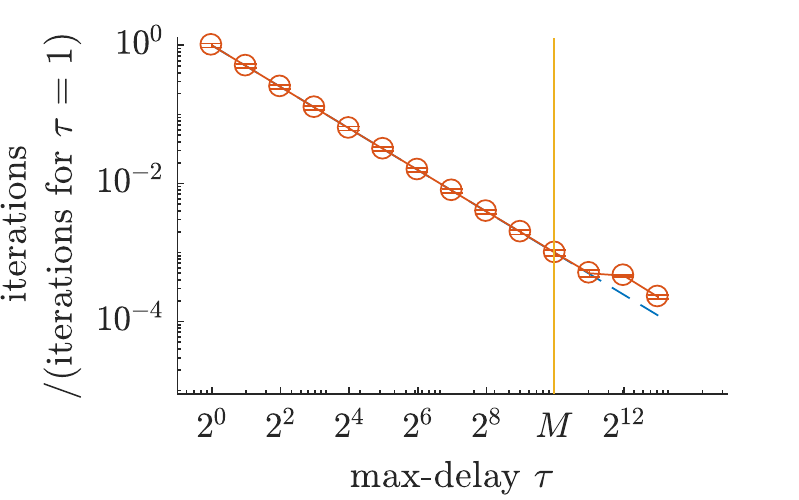}
\hfill
\includegraphics[width=0.32\linewidth]{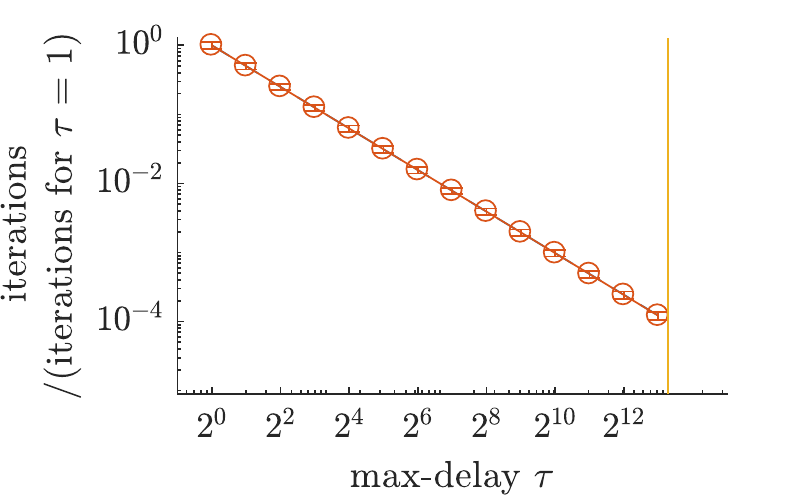}
\hfill\null
\caption{\textbf{Scaling of delayed SGD.} Parallel speedup for various delay values $\tau \in \{2^0,\dots,2^{14}\}$ and problem instances with $M \in \{0,1,10\}$ (top) and $M \in \{100,1000,10000\}$ (bottom), on the synthetic optimization problem described in Section \ref{ssec:synth}, averaged over three random seeds (depicting mean and $\pm$SD). Plots depict number of iterations (i.e.\ parallel running time $\frac{1}{b} T(b,\epsilon)$), normalized by $T(1,\epsilon)$, required to reach the target accuracy with tuned optimal learning rates.}
\label{fig:scaling_worst}
\end{figure}

\begin{figure}[h!]
\hfill
\includegraphics[width=0.32\linewidth]{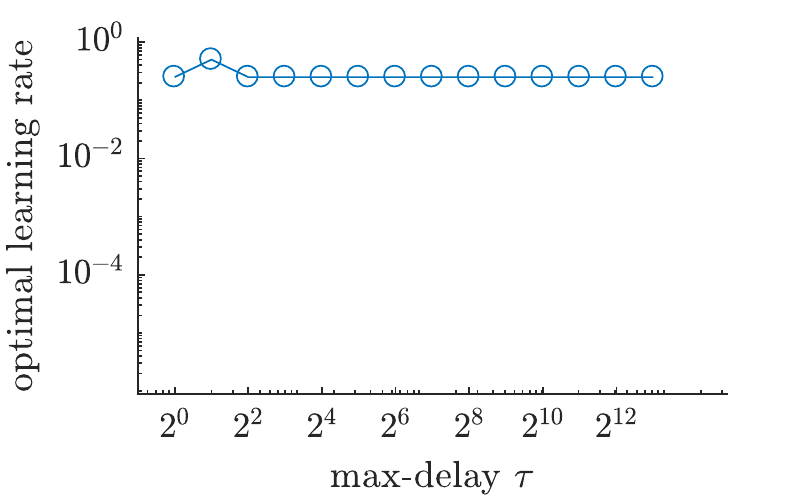}
\hfill
\includegraphics[width=0.32\linewidth]{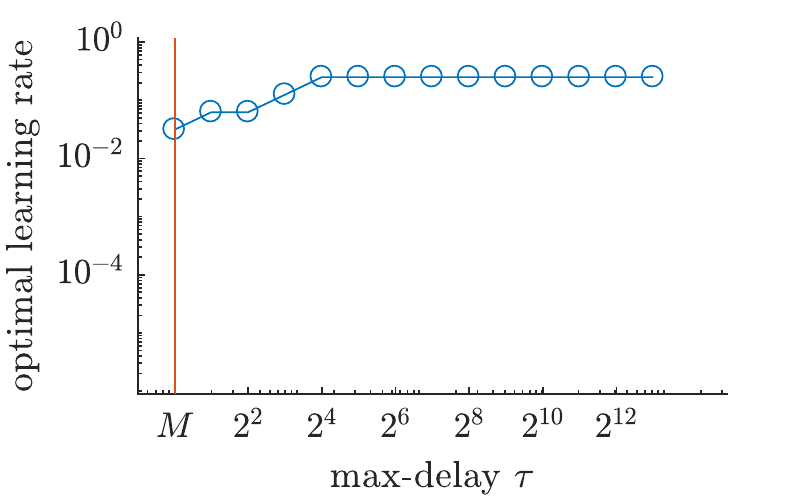}
\hfill
\includegraphics[width=0.32\linewidth]{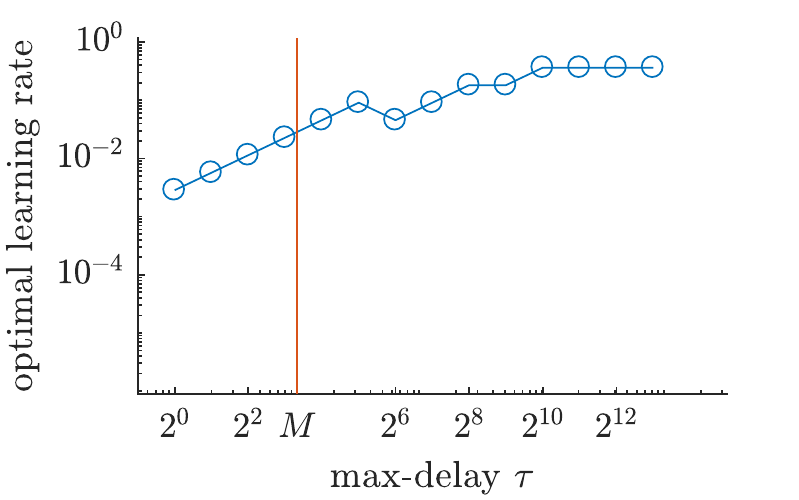}
\hfill\null

\hfill
\includegraphics[width=0.32\linewidth]{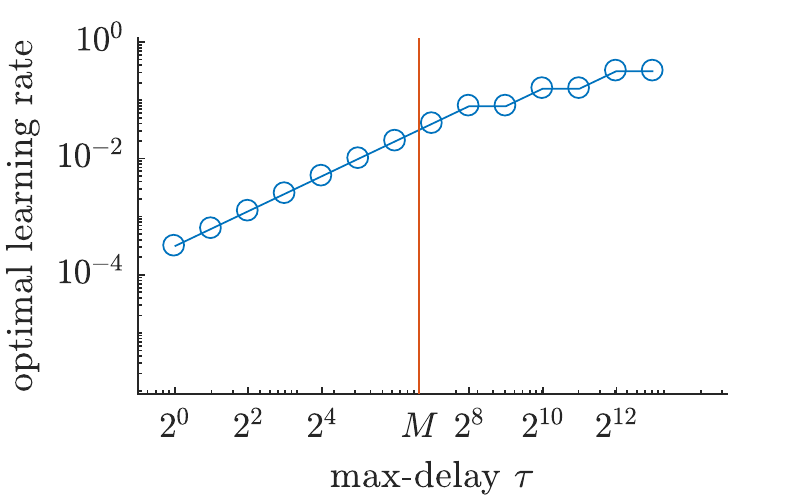}
\hfill
\includegraphics[width=0.32\linewidth]{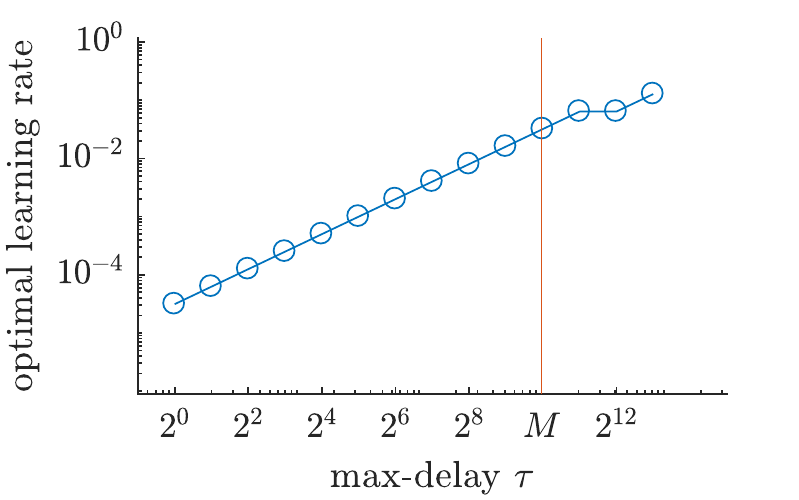}
\hfill
\includegraphics[width=0.32\linewidth]{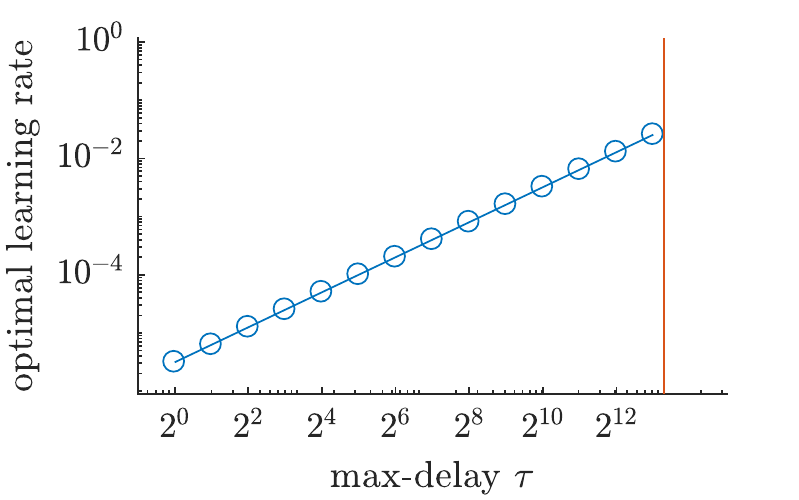}
\hfill\null
\caption{\textbf{Optimal learning rate $\gamma_{\rm d}$ for delayed SGD} for the results reported in Figure~\ref{fig:scaling_worst}.}
\label{fig:lr_worst}
\end{figure}

\subsection{Hyperparameters for Deep Learning Experiments}
For Figures~\ref{fig:speedup-m} and  \ref{fig:speedup-m-resnet18}, we tune the learning rate for each batch size. In particular, for ResNet-8, we use as step size, $0.4$ when batch size is $32$, $0.2$ when batch size is $64$ and $0.05$ for all other batch sizes. The step size was chosen from the set $\{0.005, 0.05, 0.02, 0.1, 0.2, 0.4\}$. 

For ResNet-18, we use as step size, $0.02$ when batch size is $32$, $0.04$ when batch size is $64$ and $0.1$ for all other batch sizes. The step size was chosen from the set $\{0.005, 0.02, 0.05, 0.1\}$.

\end{document}